\documentclass[aps,floatfix,longbibliography,prx,reprint,superscriptaddress]{revtex4-1}

\usepackage{amsmath,amssymb,amsthm}
\usepackage{bbm}  
\usepackage{graphicx}
\PassOptionsToPackage{hyphens}{url}
\usepackage{hyperref}
\usepackage[T1]{fontenc}
\usepackage[utf8x]{inputenc}
\usepackage{nicefrac}
\usepackage{verbatim}
\usepackage{tabularx}
\usepackage{xcolor}

\def\ind{{\mathbbm{1}}}

\newtheorem{theorem}{Property}[section]
\newtheorem{lemma}[theorem]{Lemma}
\newtheorem{corollary}[theorem]{Corollary}

\DeclareMathOperator*{\argmax}{arg\,max}

\def\mC{{\bm{C}}}
\def\mF{{\bm{F}}}
\def\mH{{\bm{H}}}
\def\mW{{\bm{W}}}
\def\mX{{\bm{X}}}
\def\mZ{{\bm{Z}}}
\def\vc{{\bm{c}}}
\def\vf{{\bm{f}}}
\def\vv{{\bm{v}}}
\def\vw{{\bm{w}}}
\def\vx{{\bm{x}}}
\def\vtheta{{\bm{\theta}}}
\def\vphi{{\bm{\phi}}}

\DeclareMathAlphabet{\mathpzc}{OT1}{pzc}{m}{it}
\def\cala{{\mathpzc{a}}}
\def\calb{{\mathpzc{b}}}
\def\calc{{\mathpzc{c}}}

\newcommand{\ussqR}{\frac{1}{\sqrt{D}}}
\newcommand{\ussqN}{\frac{1}{\sqrt{N}}}
\newcommand{\usR}{\frac{1}{D}}

\newcommand{\ussD}{\frac{1}{\sqrt{D}}}
\newcommand{\ussR}{\frac{1}{\sqrt{D}}}
\newcommand{\usN}{\frac{1}{N}}

\newcommand{\dd}{\operatorname{d}\!}
\newcommand{\EE}{\mathbb{E}\,}

\newcommand{\erf}{\mathrm{erf}}

\newcommand{\order}[1]{\mathcal{O}(#1)}

\definecolor{C0}{HTML}{1f77b4}
\definecolor{C1}{HTML}{ff7f0e}
\definecolor{C2}{HTML}{2ca02c}
\definecolor{C3}{HTML}{d62728}

\begin{document}

\title{Modelling the influence of data structure\\ on learning in neural
  networks: the Hidden Manifold Model}

\author{Sebastian Goldt}
\email{sebastian.goldt@phys.ens.fr}
\affiliation{Laboratoire de Physique de l’Ecole Normale Supérieure, Université
  PSL, CNRS, Sorbonne Université, Université Paris-Diderot, Sorbonne Paris
  Cité, Paris, France}
\author{Marc Mézard}
\email{marc.mezard@ens.fr}
\author{Florent Krzakala}
\email{florent.krzakala@ens.fr}
\affiliation{Laboratoire de Physique de l’Ecole Normale Supérieure, Université
  PSL, CNRS, Sorbonne Université, Université Paris-Diderot, Sorbonne Paris
  Cité, Paris, France}
\author{Lenka Zdeborov\'a}
\email{lenka.zdeborova@cea.fr}
\affiliation{Institut de Physique Théorique, CNRS, CEA, Université
  Paris-Saclay, France}

\date{\today}

\begin{abstract}
  Understanding the reasons for the success of deep neural networks trained
  using stochastic gradient-based methods is a key open problem for the
  nascent theory of deep learning.
  The types of data where these networks are most successful, such as images or
  sequences of speech, are characterised by intricate correlations. %
  Yet, most theoretical work on neural networks does not explicitly model
  training data, or assumes that elements of each data sample are drawn independently from some
  factorised probability distribution.
  These approaches are thus by construction blind to the correlation structure
  of real-world data sets and their impact on learning in neural networks.
  Here, we introduce a generative model for \emph{structured} data sets that we
  call the \emph{hidden manifold model} (HMM).
  The idea is to construct high-dimensional inputs that lie on a
  lower-dimensional manifold, with labels that depend only on their position
  within this manifold, akin to a single layer decoder or generator in a
  generative adversarial network. %
  We demonstrate that learning of the hidden manifold model is amenable to an
  analytical treatment by proving a ``Gaussian Equivalence Property'' (GEP), and
  we use the GEP to show how the dynamics of two-layer neural networks trained
  using one-pass stochastic gradient descent is captured by a set of
  integro-differential equations that track the performance of the network at
  all times. %
  This permits us to analyse in detail how a neural network learns functions of
  increasing complexity during training, how its performance depends on its size
  and how it is impacted by parameters such as the learning rate or the
  dimension of the hidden manifold.
\end{abstract}

\maketitle

\section{Introduction}

The data sets on which modern neural networks are most successful, such as
images~\cite{krizhevsky2012imagenet, lecun2015deep} or natural
language~\cite{sutskever2014sequence}, are characterised by complicated
correlations. Yet, most theoretical works on neural networks in statistics or
theoretical computer science do not model the structure of the training data at
all~\cite{Vapnik2013, Mohri2012}, which amounts to assuming that the data set is
chosen in a worst-case (adversarial) manner. A line of theoretical works
complementary to the statistics approach emanated from statistical
physics~\cite{Gardner1989,Seung1992,Engel2001,Zdeborova2016}. These works model
inputs as element-wise i.i.d.\ draws from some probability distribution, with
labels that are either random or given by some random, but fixed function of the
inputs. Despite providing valuable insights, these approaches are by
construction blind to even basic statistical properties of real-world data sets
such as their covariance structure. This lack of mathematical models for data
sets is a major impediment for understanding the effectiveness of deep neural
networks.

The structure present in realistic data sets can be illustrated well with
classic data sets for image classification, such as the handwritten digits of
MNIST~\cite{lecun1998} or the images of
CIFAR10~\cite{krizhevsky2009learning}. The inputs that the neural network has to
classify are images, so \emph{a priori} the input space is the
high-dimensional~$\mathbb{R}^N$, corresponding to the number of pixels, with~$N$
large. However, the inputs that can be recognised as actual images rather than
random noise, span but a lower-dimensional manifold within~$\mathbb{R}^{N}$, see
Fig.~\ref{fig:intuition}. This manifold hence constitutes the actual input
space, or the ``world'', of our problem. While the manifold is not easily
defined, it is tangible: for example, its dimension can be estimated based on
the neighbourhoods of inputs in the data set~\cite{Grassberger1983, Costa2004,
  Levina2004a, Spigler2019}, and was found to be around $D\approx14$ for MNIST,
and $D\approx35$ for CIFAR10, compared to $N=784$ and $N=3072$,
respectively. 
We will call inputs \emph{structured} if they are concentrated on a
lower-dimensional manifold and thus have a lower-dimensional latent
representation, which consists of the coordinates of the input within that
manifold. 

A complementary view on the data manifold is provided by today's most powerful
generative models, called Generative Adversarial Networks
(GAN)~\cite{goodfellow2014generative}. A GAN~$\mathcal{G}$ is a neural network
that is trained to take random noise as its input and to transform it into
outputs that resemble a given target distribution. For example, GANs can
generate realistic synthetic images of human
faces~\cite{radford2015unsupervised,
  karras2019analyzing}. 
From this point of view, the mapping from the hidden manifold to the input space
is given by the function that the GAN $\mathcal{G}$ computes.

\begin{figure}
  \centering
  \includegraphics[width=\linewidth]{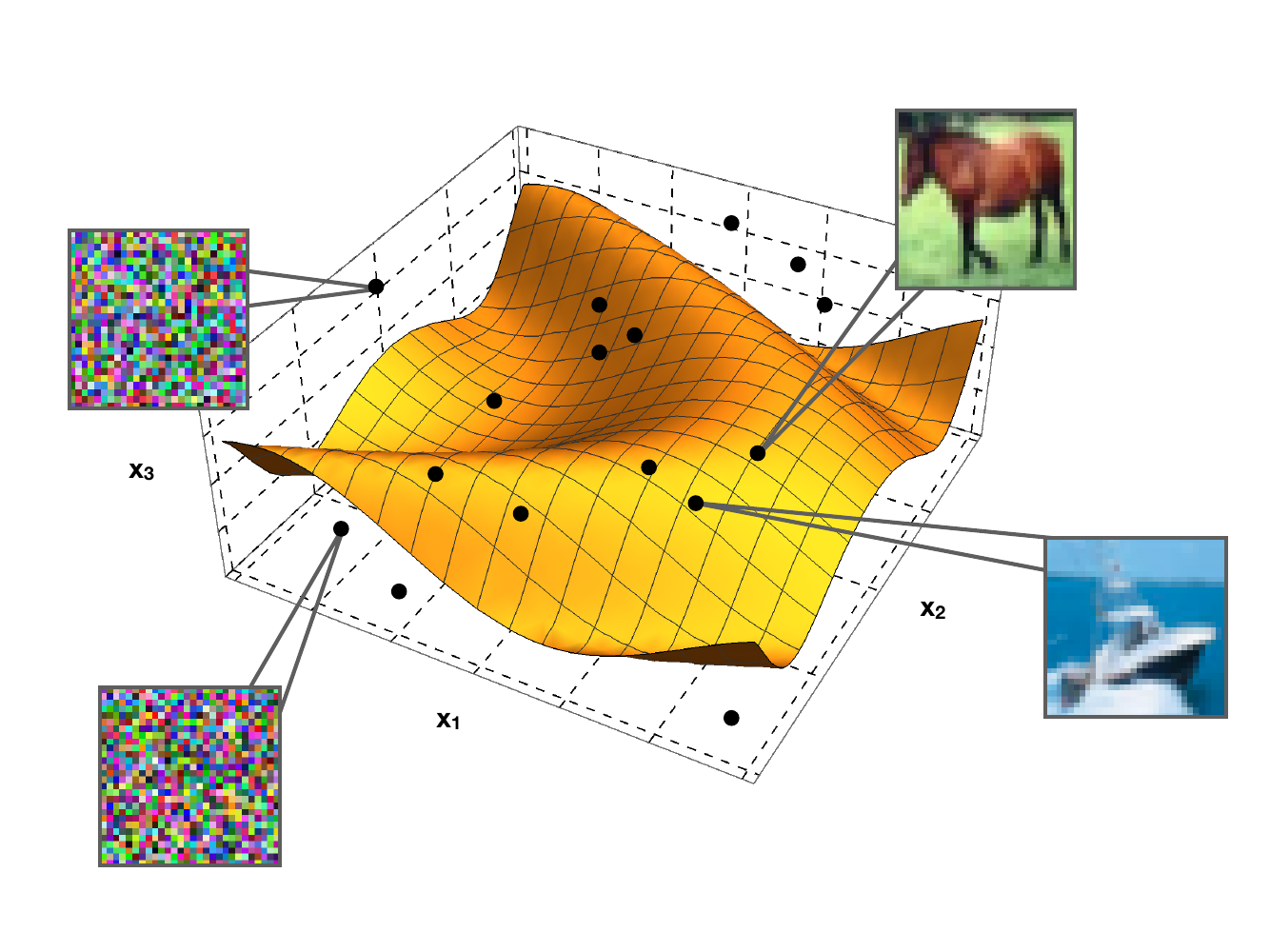}
  \caption{\label{fig:intuition} We illustrate the notion of a \textbf{hidden
      manifold in input space} using CIFAR10 example images. Each black point
    indicates a possible input in a high-dimensional input space
    $\mathbb{R}^N$. Most points in this space cannot be interpreted as images at
    all; however, those points that can be interpreted as real images tend to
    concentrate on a lower-dimensional manifold, here sketched as a
    two-dimensional curved surface in a three-dimensional space. The intrinsic
    dimension $D$ of these lower-dimensional manifolds has been measured
    numerically~\cite{Grassberger1983, Costa2004, Levina2004a, Spigler2019}.}
\end{figure}

\subsection{Main results}

In this paper, and specifically in Sec.~\ref{sec:hmm}, we introduce a generative
model for structured data sets in the above sense that we call the \emph{hidden
  manifold model} (HMM)~\footnote{Not to be confused with the hidden Markov
  models that are also commonly abreviated as HMM, but which are not used in the
  present paper.}. It is a generative model that produces tuples $(\vx, y^*)$ of
high-dimensional inputs $\vx\in\mathbb{R}^N$ and their scalar labels $y^*$. The
key idea is to construct the inputs such that they lie on a lower-dimensional
manifold; their labels are then a function of only their position within that
manifold. The way the inputs are generated is akin to a learnt single layer
decoder with random inputs; it can also be viewed as a single layer generator
neural network of a learnt GAN.  As a result, inputs drawn from the HMM have
non-trivial correlations, do not follow a normal distribution, and their labels
$y^*$ cannot be written as a simple function of the inputs $\vx$.

Our key theoretical result, presented in Sec.~\ref{sec:solution}, is to show
that despite these correlations, the HMM is amenable to an analytical treatment
in a thermodynamic limit of large dimensions $N$ and $D$, large number of
samples $P$, and fixed respective ratios $D/N$, $P/N$. We derive the solution by
first demonstrating a ``Gaussian Equivalence Property'' or GEP
(Proposition~\ref{gep}). As a first application, we use the GEP to derive a set
of integro-differential equations that captures the behaviour of two-layer
neural networks, with $K=\order{1}$ hidden units, trained using stochastic
gradient descent. These equations extend the classical analysis of the dynamics
of two-layer neural networks on unstructured data~\cite{Kinzel1990, Biehl1995,
  Saad1995a, Saad1995b} to the hidden manifold and provide detailed insight into the
dynamics of learning.

We then use these equations to study the dynamics and the performance of
two-layer neural networks trained on data generated by the HMM, in
Sec.~\ref{sec:results}. We find back the specialisation of hidden units, known
from the canonical teacher-student model. We analyse the learning for different
feature matrices, and show that Hadamard matrices perform slightly better than
i.i.d. Gaussian ones. We show analytically that the generalisation performance
deteriorates as the manifold dimension $D$ grows. We show that the learning rate
has a very minor influence on the asymptotic error, and analyse how the error
final error of the network changes as a function of the width of the hidden
layer.

Sec.~\ref{sec:comparison} is devoted to comparison of learning on the HMM and on
real data sets such as MNIST~\cite{lecun1998},
fashion-MNIST~\cite{xiao2017online} or CIFAR10~\cite{krizhevsky2009learning}.
In particular, we demonstrate that neural networks learn functions of increasing
complexity over the course of training on both the HMM and real data sets. We
also compare the memorisation of some samples during the early stages of
training between the HMM to real data. These comparisons provide strong evidence
that the HMM captures the properties of learning with one-pass SGD and two-layer
neural networks on some of the standard benchmark data sets rather faithfully.


\subsection{Further related work}

\paragraph*{The need for models of structured data.} Several works have
recognised the importance of the structure in the data sets used for machine
learning, and in particular the need to go beyond the simple component-wise
i.i.d.\ modelling~\cite{derrida1987exactly, bruna2013invariant, PatelNIPS2016,
  mezard2017mean, Gabrie2018, Mossel2018}. While we will focus on the ability of
neural networks to generalise from examples, two recent papers studied a
network's ability to \emph{store} inputs with lower-dimensional structure and
random labels: \citet{Chung2018} studied the linear separability of general,
finite-dimensional manifolds and their interesting consequences for the training
of deep neural networks~\cite{chung2018learning,cohen2020separability},
while~\citet{Rotondo2019a} extended Cover's classic argument~\cite{Cover1965} to
count the number of learnable dichotomies when inputs are grouped in tuples of
$k$ inputs with the same label. Recently, \citet{yoshida2019datadependence}
analysed the dynamics of online learning for data having an arbitrary covariance
matrix, finding an infinite hierarchy of ODEs. Their study implicitly assumes
that inputs have a Gaussian distribution, while our approach handles a more
general data structure.  The importance of the spectral properties of the data
was recognised for learning in deep neural networks by~\citet{Saxe2014} in the
special case of \emph{linear} neurons, where the whole network performs a linear
transformation of the data.

\paragraph*{Relation to random feature learning.} The hidden manifold model has
an interesting link to random feature learning with unstructured i.i.d.\ input
data. The idea of learning with random features goes back to the mechanical
perceptron of the 1960s~\cite{Rosenblatt1962} and was extended into the random
kitchen sinks of~\citet{rahimi2008random, rahimi2009weighted}. Remarkably,
random feature learning in the same scaling limit as used in the theoretical
part of this paper was analysed in several recent and concurrent works, notably
in~\cite{louart2018random, mei2019generalization} for ridge regression, and
in~\cite{montanari2019generalization} for max-margin linear classifiers. These
papers consider full batch learning, i.e.\ all samples are used at the same
time, which makes one difference from our online (one-pass stochastic gradient
descent) analysis. Another important difference is that we study learning in a
neural network with two layers of learned weights, while the existing works
study simpler linear (perceptron-type) architectures where only one layer is
learned. Perhaps more importantly, in our analysis the features do not need to
be random, but can be chosen deterministically or even be learnt from data using
a GAN or an autoencoder. The principles underlying the analytic solution of this
paper as well as~\cite{louart2018random, mei2019generalization,
  montanari2019generalization} rely on the Gaussian Equivalence Property, which
is stated and used independently in those papers.

\paragraph*{Gaussian equivalence and random matrix theory.} Special cases of the
Gaussian Equivalence Property were in fact derived previously using random matrix
theory in~\cite{hachem2007deterministic, cheng2013spectrum, fan2019spectral,
  pennington2017nonlinear}, and this equivalent Gaussian covariates mapping was
explicitly stated and used in~\cite{mei2019generalization,
  montanari2019generalization}. This formulation has recently been extended to a
broader setting of concentrated vectors encompassing data coming from a GAN
in~\cite{seddik2019kernel, seddik2020random}, a version closer to our
formulation.

\subsection{Reproducibility}

We provide the full code to reproduce our experiments as well as an integrator
for the equations of motion of two-layer networks online~\cite{code}.

\subsection{Learning setup}
\label{sec:setup}

This paper focuses on the dynamics and performance of fully-connected two-layer
neural networks with $K$ hidden units and first- and second-layer weights
$\mW\in\mathbb{R}^{K\times N}$ and $\vv\in\mathbb{R}^K$, resp. Given an input
$\vx\in\mathbb{R}^N$, the output of a network with parameters
$\vtheta=(\mW, \vv)$ is given by
\begin{equation}
  \label{eq:phi}
  \phi(\vx; \vtheta) = \sum_k^K v_k g\left(\vw_k \vx/\sqrt{N}\right),
\end{equation}
where $\vw_k$ is the $k$th row of $\mW$, and $g: \mathbb{R} \to \mathbb{R}$ is
the non-linear activation function of the network, acting component-wise. We
study sigmoidal and ReLU networks with $g(x)=\erf(x/\sqrt{2})$ and
$g(x)=\max(0, x)$, resp.

We will train the neural network on data sets with $P$ input-output pairs
$(\vx_\mu, y_\mu^*)$, $\mu=1,\ldots,P$, where we use the starred $y_\mu^*$ to
denote the \emph{true} label of an input $\vx_\mu$. Networks are trained by
minimising the quadratic training error
$E(\vtheta)= 1/2 \sum_{\mu=1}^P \Delta_\mu^2$ with
$\Delta_\mu = \left[ \phi(\vx_\mu, \vtheta) - y_\mu^*\right]$ using stochastic
(one-pass, online) gradient descent (SGD) with constant learning rate~$\eta$ and
mini-batch size 1,
\begin{equation}
  \label{eq:sgdtheta}
  \vtheta_{\mu+1} = \vtheta_\mu - \eta \nabla_{\vtheta} E(\theta)
  |_{\vtheta_\mu, \vx_\mu, y^*_\mu}.
\end{equation}
Initial weights for both layers were always taken component-wise i.i.d.\ from
the normal distribution with mean 0 and standard deviation $10^{-3}$.

The key quantity of interest is the \emph{test error} or \emph{generalisation
  error} of a network, for which we compare its predictions to the labels given
in a test set that is composed of $P^*$ input-output pairs $(\vx_\mu, y^*_\mu)$,
$\mu=1,\ldots,P^*$ that are \emph{not} used during training,
\begin{equation}
  \epsilon_g(\vtheta) \equiv \frac{1}{2P^*} \sum_\mu^{P^*} {\left[ \phi(\vx_\mu, \vtheta) -
      y_\mu^* \right]}^2.
\end{equation}
The test set in our setting is generated by the same
probabilistic model that generated the training data.

\subsubsection{The canonical teacher-student model}
\label{sec:datasets}

The joint probability distribution of input-output pairs $(\vx_\mu, y^*_\mu)$ is
inaccessible for realistic data sets such as CIFAR10, preventing analytical
control over the test error and other quantities of interest. To make
theoretical progress, it is therefore promising to study the generalisation
ability of neural networks for data arising from a probabilistic generative
model.

A classic model for data sets is the \emph{canonical teacher-student setup}
where inputs $\vx_\mu$ are drawn element-wise i.i.d.\ from the standard normal
distribution and labels are given by a random, but fixed, neural network with
weights $\vtheta^*$ acting on the inputs: $y^*_\mu = \phi(\vx_\mu,
\vtheta^*)$. The network that generates the labels is called the teacher, while
the network that is trained is called the student. The model was introduced by
Gardner \& Derrida~\cite{Gardner1989}, and its study has provided many valuable
insights into the generalisation ability of neural networks from an average-case
perspective, particularly within the framework of statistical
mechanics~\cite{Seung1992, Watkin1993, Engel2001, Zdeborova2016, Advani2017,
  Aubin2018, barbier2019optimal, Goldt2019b, Yoshida2019a}, and also in recent
works in theoretical computer science, e.g.~\cite{ge2017learning,
  li2017convergence, mei2019generalization, Arora2019}. However, it has the
notable shortcoming that its analysis crucially relies on the fact that inputs
are i.i.d.\ Gaussians and hence uncorrelated.

\section{The Hidden Manifold Model}
\label{sec:hmm}

\begin{figure}
  \centering
  \includegraphics[width=\linewidth]{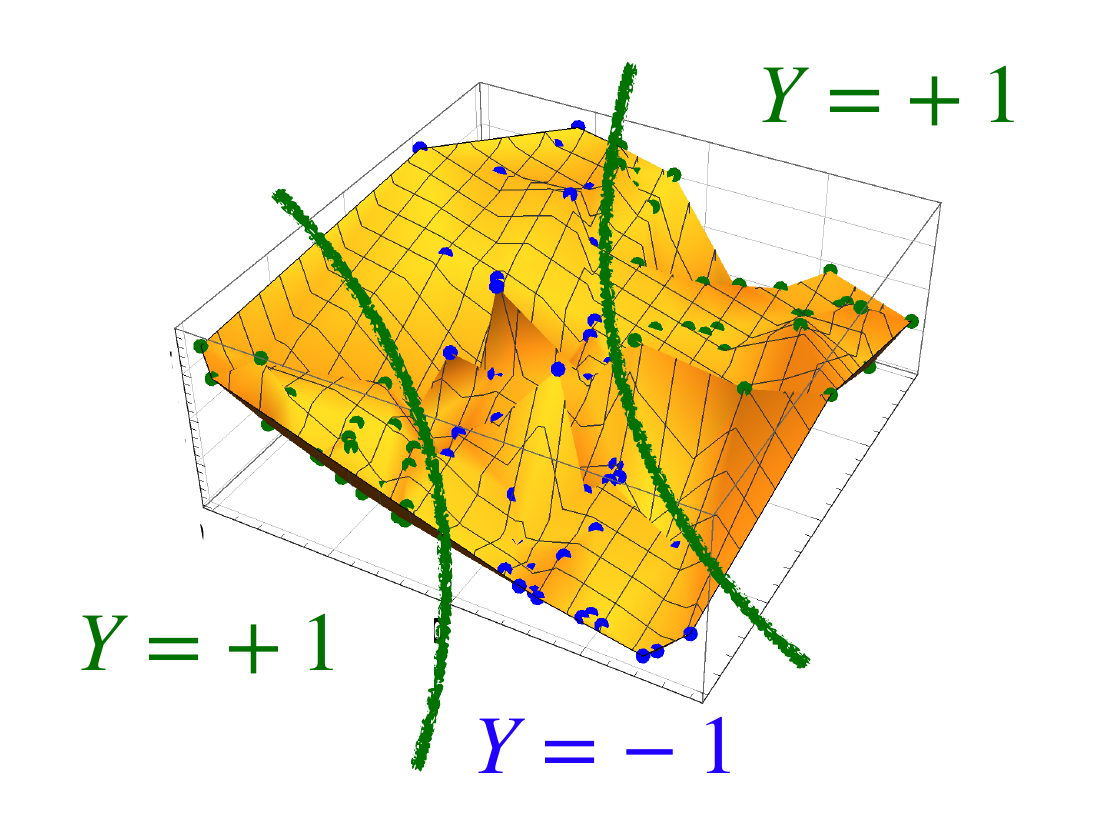}
  \caption{\label{fig:hmm_example} \textbf{The hidden manifold model} proposed
    here is a generative model for structured data sets, where inputs
    $\vx$, Eq.~\eqref{eq:structured-inputs} (blue and green balls), concentrate on a
    lower-dimensional manifold in input space (yellow surface). Their labels
    $y^*$ is a function of their position on the manifold; here we show the
    setup of a classification task with two classes~$y^*=\pm 1$. In
    our analysis the labels are generated according to Eq.~\eqref{eq:teacher2}.}
\end{figure}

We now introduce a new generative probabilistic model for structured data sets
with correlations. To generate a data set containing $P$ inputs in $N$
dimensions, we first choose $D$ feature vectors $ \vf_r$, $r=1,\dots,D$. These
are vectors in $N$ dimensions and we collect them in a feature matrix
$\mF\in\mathbb{R}^{D\times N}$. Next we draw $P$ vectors $\vc_\mu$ with random
i.i.d.\ components drawn from the normal distribution with mean zero and unit
variance and collect them in the matrix $\mC\in\mathbb{R}^{P \times D}$. The
vector $\vc_\mu$ gives the coordinates of the $\mu$th input on the
lower-dimensional manifold spanned by the feature vectors in $\mF$. We will call
$\vc_\mu$ the \emph{latent representation} of the input $\vx_\mu$, which is
given by the $\mu$th row of
\begin{equation}
  \label{eq:structured-inputs}
  \mX = f\left(\mC\mF/\sqrt{D}\right) \in\mathbb{R}^{P\times N},
\end{equation}
where $f$ is a non-linear function acting component-wise. In this model, the
``world'' of the data on which the true label can depend is a $D$-dimensional
manifold, which is obtained from the linear span of $\mF$ through a ``folding''
process induced by the nonlinear function $f$. We note that the structure of
data of the same type arises in a learned variational autoencoder
network~\cite{kingma2013auto} with single layer, or in a learned GAN
network~\cite{goodfellow2014generative} with a single layer generator network,
the matrix $C$ then corresponds to the random input, the $F$ to the learned
features, $f$ is the corresponding output activation. The matrix $\mF$ can be
generic with a certain normalisation, such that its elements are
$\order{1}$. For our analysis to be valid, we will later assume the normalisation given in
Eq.~\eqref{eq:F-cond} and balance condition given by
\eqref{eq:balance}, other than that our analysis hold for arbitrary
matrices $\mF$.

The labels are obtained by applying a two-layer neural network with
weights
$\widetilde{\vtheta}=(\widetilde{\mW}\in\mathbb{R}^{M\times D},
\widetilde{\vv}\in\mathbb{R}^M)$ within the unfolded hidden manifold according to
\begin{equation}
  \label{eq:teacher2}
  {y}_\mu^* = \phi(\vc_\mu, \widetilde{\vtheta}) =
  \sum_m^M \tilde{v}^m \tilde g\left(\widetilde{\vw}^m \vc_\mu/\sqrt{D}\right).
\end{equation}
We draw the weights in both layers component-wise i.i.d.\ from the normal
distribution with unity variance, unless we note it otherwise. The key point
here is the dependency of labels ${y}^*_\mu$ on the coordinates of the
lower-dimensional manifold $\vc_\mu$ rather than on the high-dimensional data
$\vx_\mu$. We expect the exact functional form of this dependence not to be
crucial for the empirical part of this work, and that there are other forms that
would present the same behaviour. Notably it would be interesting to consider
ones where the latent representation is conditioned to the labels as in
conditional GANs~\cite{michelsanti2017conditional} or the manifold model
of~\cite{cohen2020separability}.

\section{The solution of the Hidden Manifold Model}
\label{sec:solution}

\subsection{The Gaussian Equivalence Property}
\label{sec:GEP}

The difficulty in analysing HMM comes from the fact that the various components
of one given input pattern, say $x_{\mu i}$ and $x_{\nu j}$, are
correlated. Yet, a key feature of the model is that it is amenable to an
analytical treatment. To that end, we will be studying the standard
thermodynamic limit of the statistical physics of learning where the size of the
input space $N\to\infty$, together with the number $P\to \infty$ of patterns
that are presented for learning, while keeping the ratio $\alpha\equiv P/N$ fixed.  In
statistics this corresponds to the challenging high-dimensional limit. The
hidden manifold model can then be studied analytically if one assumes that the
latent dimension $D$, i.e.\ the dimension of the feature space, also scales with
$N$, meaning that it goes to $\infty$ with a fixed ratio $\delta\equiv D/N$
which is of order 1 with respect to $N$, so that we have
\begin{equation}
  \label{eq:thermolimit}
  N,P,D\to \infty, \quad \text{with fixed}~ \alpha \equiv \frac{P}{N} \;\text{and}\; \delta\equiv\frac{D}{N}.
\end{equation}
In this limit, the relevant variables are the ``local fields'' or
pre-activations that are acted upon by the neurons in the hidden layer. They can
be shown to follow a Gaussian distribution in the thermodynamic
limit~\eqref{eq:thermolimit}.  We will now make this statement precise by
formulating a ``Gaussian Equivalence Property'' (GEP). We will demonstrate the
power of this equivalence by deriving a set of exact equations for online
learning in Sec.~\ref{sec:sgd}.

\subsubsection{Statement of the Property}

Let $\{C_r\}_{r=1}^D$ be $D$ i.i.d.\ Gaussian random variables distributed as
${\cal N}(0,1)$.  In the following we shall denote by $\EE$ the expectation
value with respect to this distribution.  Define $N$ variables $u_i$,
$i=1,\dots,N$ as linear superpositions of the $C_r$ variables,
\begin{equation}
  \label{eq:u}
  u_i  \equiv \ussqR\sum_{r=1}^D C_r F_{ir},
\end{equation}
and $M$ variables $\nu^m$, $m=1,\dots, M$ as other linear
superpositions,
\begin{equation}
  \label{eq:nu}
  \nu^m  \equiv \ussqR \sum_{r=1}^D C_r \tilde w^m_r,
\end{equation}
where $\tilde w^m_r$ are the teacher weights Eq.~(\ref{eq:teacher2}). Define $K$ variables
$\lambda^k$ as linear superpositions of $f(u_i)$ where $f$ is an
arbitrary function:
\begin{equation}
  \label{eq:lambda}
  \lambda^k \equiv \ussqN \sum_{i=1}^N w_i^k f(u_i),
\end{equation}
where $\tilde w^k_i$ are the student weights Eq.~(\ref{eq:phi}). We will occasionally
write $(\lambda, \nu)$ to denote the tuple of all local fields $\lambda^k$ and $\nu^m$.
Denoting by $\langle g(u)\rangle$ the expectation of a function $g(u)$ when $u$ is a
normal variable with distribution $u\sim{\cal N}(0,1)$, we also introduce for
convenience the ``centered'' variables
\begin{equation}
  \label{eq:lambda_centered}
  \tilde \lambda^k \equiv \ussqN \sum_{i=1}^N w_i^k (f(u_i) -
  \langle  f(u) \rangle ) .
\end{equation}
Notice that our notations keeps upper indices for indices which take values in a
finite range ($k,\ell \in\{1,\dots,K\}$, $m,n \in\{1,\dots,M\}$), and lower
indices for those which have a range of order $N$ ($i,j \in\{1,\dots,N\}$;
$r,s \in\{1,\dots,D\}$).

As the $C_r$ are Gaussian, the $u_i$ variables are also Gaussian variables, with
mean zero and a matrix of covariance
\begin{equation}
  U_{ij} = \EE [ u_i u_j ]  = \usR \sum_{r=1}^D F_{ir}F_{jr}\ .
\end{equation}
Note that the covariances of the $u_i$ variables scale in the thermodynamic
limit as
\begin{equation}
  \EE [u_i^2]=1; \quad \EE[u_i u_j]=\order{1/\sqrt{D}}, \; i\neq j.
\end{equation}
We assume that, in the thermodynamic limit, the $\mW$, $\tilde \mW$ and $\mF$
matrices have elements of $\order{1}$ and that for $i\neq j$,
\begin{equation}
  \label{eq:F-cond}
  \frac{1}{\sqrt{D}}\sum_{r=1}^D F_{ir}F_{jr}=\mathcal{O}(1), \quad \mathrm{and}\quad   \sum_{r=1}^D (F_{ir})^2=D. 
\end{equation}
Notice that the only variables which are drawn i.i.d.\ from a Gaussian
distribution are the coefficients $C_r$. Most importantly, the matrices $\mF$
and $\mW$ can be arbitrary (and deterministic) as long as they are ``balanced''
in the sense that
$\forall p,q \ge 1, \; \forall k_1,\ldots, k_p,r_1,\ldots r_q$, we have
\begin{multline}
  \label{eq:balance}
  S^{k_1k_2\ldots k_p}_{r_1r_2\ldots r_q}\\=\frac{1}{\sqrt{N}}\sum_i
  w_i^{k_1}w_i^{k_2}\ldots w_i^{k_p}F_{ir_1}F_{ir_2}\ldots F_{ir_q}=\order{1},
\end{multline}
with the $q$ and $p$ distinct. We also have a similar scaling for the
combinations involving the teacher weights $\tilde w^m_r$. This is the key
assumption behind the Gaussian Equivalence Property, and we will discuss its
interpretation immediately after the statement of the GEP:
\begin{theorem}
  \label{gep}
  \textbf{Gaussian Equivalence Property (GEP)} In the asymptotic limit when
  $N\to \infty$, $D\to\infty$, with $K,M$ and the ratio $D/N$ finite, and under
  the assumption~\eqref{eq:balance}, $\{\lambda^k\}$ and $\{\nu^m\}$ are $K+M$
  jointly Gaussian variables, with mean
  \begin{equation}
    \EE [\lambda^k] =a \ussqN
    \sum_{i=1}^N w_i^k; \quad \EE [\nu^m]=0,
  \end{equation}
  and covariance 
  \begin{alignat}{3}
    & Q^{k\ell} && \equiv \EE [\tilde \lambda^k \tilde \lambda^\ell] &&
    =(c-a^2-b^2)
    W^{k\ell}+b^2 \Sigma^{k\ell} ,\label{eq:Q}  \\
    & R^{km} && \equiv \EE [\tilde \lambda^k \nu^m]&&= b \usR \sum_{r=1}^D
    S_r^k\tilde w_r^m ,  \label{eq:R}    \\
    & T^{mn} && \equiv \EE [ \nu^m \nu^n] &&= \usR \sum_{r=1}^D \tilde w_r^m
    \tilde w_r^n ,\label{eq:T}
  \end{alignat}
  The ``folding function'' $f(\cdot)$ appears through the three coefficients
  $a,b,c$, which are defined as
  \begin{equation}
    \label{eq:abc}
    a \equiv \langle f(u) \rangle, \quad  b \equiv \langle u f(u) \rangle, \quad
    c \equiv \langle {f(u)}^2\rangle
  \end{equation}
  where $\langle \psi(u)\rangle$ denotes the expectation value of the function
  $\psi$ when $u \sim {\cal N}(0,1)$ is a Gaussian variable.

  The covariances are defined in terms of the three matrices
  \begin{align}
    S_r^k & \equiv\ussqN\sum_{i=1}^N w_i^k F_{ir} , \label{eq:S}\\
    W^{k\ell}& \equiv\usN\sum_{i=1}^N w_i^k w_i^\ell, \label{eq:W}\\
    \Sigma^{k\ell}&\equiv \usR \sum_{r=1}^D S_r^k S_r^\ell ,\label{eq:Sigma}
  \end{align}
  whose elements are assumed to be of order $\order{1}$ in the asymptotic limit.
  The derivation of the property is given in Appendix~\ref{app:GEP}.
\end{theorem}
In Sec.~\ref{sec:sgd}, we will see that the GEP allows us to develop an
analytical understanding of learning with the hidden manifold model. We first
discuss several aspects of the GEP in detail.

\subsubsection{Discussion}

The Gaussian Equivalence Property states that the local fields $(\lambda, \nu)$
follow a joint normal distribution if the weights of the student fulfil the
balance
condition~\eqref{eq:balance}. 
In the simplest case, where the $x_i$ are element-wise i.i.d.\ Gaussian, joint
Gaussianity of~$(\lambda, \nu)$ follows immediately from the central limit
theorem (CLT). The CLT can also be applied directly when input vectors $\vx$ are
drawn from a multi-dimensional Gaussian with fixed covariance matrix, which is
the setup of~\citet{yoshida2019datadependence}, In the Hidden Manifold Model
considered here, the inputs $x_i=f(u_i)$ and thus the~$x_i$ are not normally
distributed and have a non-trivial covariance matrix.  The GEP can be seen as a
central limit theorem for sums of weakly correlated random variables, i.e.\
$\lambda^k \sim \sum_i w^k_i x_i$. In this case, the GEP establishes that
$\lambda^k$ is Gaussian, provided that the weights of the student $w_i^k$ do not
align ``too much'' with the weights of the generator $F_{ir}$. More precisely,
we require that a sum such as
$S^k_r = \nicefrac{1}{\sqrt{N}} \sum_i w^k_i F_{ir}$ remains of order 1 in the
thermodynamic limit~\eqref{eq:thermolimit}. The balance condition is a
generalisation of this idea to the higher-order tensors defined in
Eq.~\eqref{eq:balance}.

The expansion from the hidden manifold in~$\mathbb{R}^D$ to the input
space~$\mathbb{R}^N$ can equivalently be seen as a noisy transformation of the
latent variables $\mC$. As far as the local fields~$(\lambda, \nu)$ are
concerned, we can replace the data matrix
$\mX = f\left(\nicefrac{\mC\mF}{\sqrt{D}}\right)$ with the matrix
\begin{equation}
  \label{eq:tildeX}
  \tilde{\mX} \simeq a \ind + b \, \mC \mF + (c-a^2-b^2) \mZ,
\end{equation}
where $\mZ$ is a $P\times N$ matrix with entries drawn i.i.d.\ from the normal
distribution and $\ind$ is a matrix of the same size as $\mX$ with all entries
equal to one. We use the symbol $\simeq$ here to emphasise that the two matrices
on the left and right-hand side have matching first and second moments, and are
hence equivalent in terms of their low-dimensional projections, but are not the
same matrix. We can thus think of the inputs $\mX$ as a noisy transformation of
the latent variables, even without any explicit noise in
Eq.~\eqref{eq:structured-inputs}.

We could also add noise to the expansion explicitly, for example as
$\mX=f\left(\nicefrac{\mF \mC}{\sqrt{D}}\right) + \zeta$ or
$\mX=f\left(\nicefrac{\mF (\mC + \zeta)}{\sqrt{D}}\right)$, where $\zeta$ would
be a noise matrix of appropriate dimensions. These noise injections would indeed
make the data high-dimensional, or, if added directly to the latent variables
$\mC$, result in correlated noise in the input space. In all these cases, the
GEP applies and guarantees that the noise $\zeta$ would only change the
variance of the noise term $\mZ$ that appears after application of the
GET~\eqref{eq:tildeX}. Our results are thus robust to the injection of
additional noise.

Finally, the GEP shows that there is a whole family of activation functions
$f(x)$ (those that have the same values for $a,b$ and~$c$ from Eq.~\ref{eq:abc})
that will lead to equivalent analytical results for the learning curves studied
in this paper.

\paragraph*{Related results in random matrix theory.} A related result to the
Gaussian Equivalence Property was in fact known in random matrix
theory~\cite{hachem2007deterministic, cheng2013spectrum,
  pennington2017nonlinear, louart2018random, seddik2019kernel,
  fan2019spectral}. These works study quantities that can be written as an
integral over the spectral density of the distribution of inputs $x$, such as
the test and training errors for linear regression problem. However, this
spectral density is inaccessible analytically for realistic data. The key idea
is then to re-write these integrals by replacing the intractable spectral
density with the spectral density of a Gaussian model with matching first and
second moments. Using tools from random matrix theory (RMT), one can show that
certain integrals over both spectra coincide. This mapping was explicitly used
in~\cite{mei2019generalization, montanari2019generalization}. In order to apply
tools from RMT, these works have to assume that the weights $F$ of the generator
are random. The advantage of the formulation of the GEP above is that it does
\emph{not} require the matrix $\mF$ to be a random one, and is valid as well for
deterministic or learnt weight matrices, as long as the balanced conditions
stated in Eqs.~(\ref{eq:balance}-\ref{eq:F-cond}) hold. This allows to
generalise these mappings to the case of deterministic features using Hadamard
and Fourier matrices, such as the one used in Fastfood~\cite{le2013fastfood} or
ACDC~\cite{moczulski2016acdc} layers. These orthogonal projections are actually
known to be more effective than the purely random
ones~\cite{yu2016orthogonal}. It also allows generalisation of the analysis in
this paper for data coming from a learned GAN, along the lines
of~\cite{seddik2019kernel, seddik2020random}.  We shall illustrate this point
below by analysing the dynamics of online learning when the feature matrix $\mF$
is a deterministic Hadamard matrix (cf.~Sec.~\ref{sec:hadamard}).

\subsection{The dynamics of stochastic gradient descent for the Hidden Manifold
  Model}
\label{sec:sgd}

To illustrate the power of the GEP, we now analyse the dynamics of stochastic
gradient descent~\eqref{eq:sgdtheta} in the case of \emph{online learning},
where at each step of the algorithm $\mu=1,2,\ldots$, the student's weights are
updated according to Eq.~\eqref{eq:sgdtheta} using a previously unseen sample
$(\vx_\mu, y_\mu)$. This case is also known as one-shot or single-pass SGD. The
analysis of online learning has been performed previously for the canonical
teacher-student model with i.i.d. Gaussian inputs~\cite{Kinzel1990, Biehl1995,
  Saad1995a, Saad1995b, saad2009line}, and has recently been put on a rigorous
foundation~\cite{Goldt2019b}. Here, we generalise this type of analysis to
two-layer neural networks trained on the Hidden Manifold Model.

The goal of our analysis is to track the mean-squared generalisation error of
the student with respect to the teacher at all times,
\begin{equation}
  \label{eq:eg} \epsilon_g(\vtheta, \widetilde{\vtheta}) \equiv
  \frac{1}{2}\EE {\left[ \phi(\vx, \vtheta) - \tilde{y}^*\right]}^2,
\end{equation} where the expectation $\EE$ denotes an average over an input
drawn from the hidden manifold model, Eq.~\eqref{eq:structured-inputs}, with
label ${y}_\mu^* = \phi(\vc_\mu, \widetilde{\vtheta}^*)$ given
by a teacher network with fixed weights~$\tilde{\vtheta}^*$ acting on the latent
representation, Eq.~\eqref{eq:teacher2}. Note that the weights of both the student
and the teacher, as well as the feature matrix $F_{ir}$, are held fixed when
taking the average, which is an average only over the coefficients $c_{\mu r}$.
To keep notation compact, we focus on cases where $a=\EE f(u)=0$
in~\eqref{eq:abc}, which leads to $\tilde \lambda^k = \lambda^k$
in~\eqref{eq:lambda_centered}. A generalisation to the case where $a\neq0$ is
straightforward but lengthy.

We can make progress with the high-dimensional average over $\vx$ in
Eq.~\eqref{eq:eg} by noticing that the input $\vx$ and its latent representation
$\vc$ only enter the expression via the local fields~$\nu^m$ and~$\lambda^k$,
Eqs.~(\ref{eq:nu},~\ref{eq:lambda}):
\begin{equation}
  \label{eq:eg_preactivations}
  \epsilon_g(\vtheta, \widetilde{\vtheta}) =  \frac{1}{2}\EE {\left( \sum_{k}^K
      v^k g(\lambda^k) - \sum_{m}^M \tilde v^m \tilde g(\nu^m)\right)}^2
\end{equation}
The average is now taken over the joint distribution of local fields
$\{\lambda^{k=1,\ldots,K}, \nu^{m=1,\ldots,M}\}$. The key step is then to invoke
the Gaussian Equivalence Property~\ref{gep}, which guarantees that this
distribution is a multivariate normal distribution with covariances
$Q^{k\ell}, R^{km}$, and $T^{nm}$~(\ref{eq:Q}--\ref{eq:T}). Depending on the
choice of $g(x)$ and $\tilde g(x)$, this makes it possible to compute the
average analytically; in any case, the GEP guarantees that we can
express~$\epsilon_g(\vtheta, \widetilde{\vtheta})$ as a function of only the
second-layer weights $v^k$ and $\tilde v^m$ and the matrices
$Q^{k\ell}, R^{km}$, and $T^{nm}$, which are called \emph{order parameters} in
statistical physics~\cite{Kinzel1990, Biehl1995, Saad1995a}:
  \begin{equation}
    \label{eq:eg-order-parameters} \lim_{N, D \to\infty}
\epsilon_g(\vtheta, \widetilde{\vtheta}) =
\epsilon_g(Q^{k\ell}, R^{kn}, T^{nm}, v^k, \tilde v^m)
\end{equation} where in taking the limit, we keep the ratio
$\delta\equiv\nicefrac{D}{N}$ finite (see Eq.~\ref{eq:thermolimit}). 

\subsubsection{The physical interpretation of the order parameters}

The order parameter $R^{kn}$, defined in (\ref{eq:R}, \ref{eq:S}), measures the
similarity between the action of the $k$th student node on an input $\vx_\mu$
and the $n$th teacher node acting on the corresponding latent representation
$\vc_\mu$. In the canonical teacher-student setup, where (i) the input covariance
is simply $\EE x_i x_j=\delta_{ij}$ and (ii) labels are generated by the teacher
acting directly on the inputs $\vx$, it can be readily verified that the overlap
has the simple expression
$R^{kn} \equiv \EE \lambda^k \nu^n \sim \vw^k \tilde{\vw}^n$. It was hence
called the teacher-student overlap in the previous literature. In the HMM,
however, where teacher and student network act on different vector spaces, it is
not \emph{a priori} clear how to express the teacher-student overlap in suitable
order parameters.

The matrix $Q^{k\ell}= \left[c - b^2\right]W^{k\ell}+ b^2\Sigma^{k \ell}$
quantifies the similarity between two student nodes $k$ and $\ell$, and has two
contributions: the latent student-student overlap $\Sigma^{k\ell}$, which
measures the overlap of the weights of two students nodes after they have been
projected to the hidden manifold, and the ambient student-student overlap
$W^{k\ell}$, which measures the overlap between the vectors
$\vw^k, \vw^\ell\in\mathbb{R}^N$. Finally, we also have the overlaps of the
teacher nodes are collected in the matrix $T^{nm}$, which is \emph{not}
time-dependent, as it is a function of the teacher weights only.

\subsubsection{Statement of the equations of motion}

We have derived a closed set of equations of motion that describe the dynamics
of the order parameters~$R^{km}, \Sigma^{k\ell}, W^{k\ell}$ and~$v^k$ when the
student is trained using online SGD~\eqref{eq:sgdtheta}. We stress at this point
that in the online learning, at each step of SGD a new sample is given to the
network. The weights of the network are thus uncorrelated to this sample, and
hence the GEP can be applied at every step. This is in contrast with the
full-batch learning where the correlations between weights and inputs have to be
taken into account explicitly~\cite{montanari2019generalization}. Integrating
the equations of motion and substituting the values of the order parameters into
Eq.~\eqref{eq:eg-order-parameters} gives the generalisation error at all
times. Here, we give a self-contained statement of the equations, and relegate
the details of the derivation to Appendix~\ref{app:derivation};

A key object in our analysis is the spectrum of the matrix \begin{equation}
  \Omega_{rs}\equiv\usN \sum_i F_{ir} F_{is}.
\end{equation}
We denote its eigenvalues and corresponding eigenvectors by $\rho$ and
$\psi_\rho$, and write $p_\Omega(\rho)$ for the distribution of eigenvalues. It
turns out that it is convenient to rewrite the teacher-student overlap as an
integral over a density~$r^{km}(\rho, t)$, which is a function of $\rho$ and of
the normalised number of steps $t=P/ N$, which can be interpreted as a
continuous time-like variable. We then have \begin{equation}
  \label{eq:R_int}
  R^{km}(t)= b \;\int \dd \rho\; p_{\Omega}(\rho)\;
  r^{km}(\rho, t).
\end{equation}
with $b\equiv \langle u f(u) \rangle$~\eqref{eq:abc}. In the canonical
teacher-student model, introducing such a density and the integral that comes
with it is not necessary, but in the HMM is a consequence of the non-trivial
correlation matrix $\EE x_i x_j$ between input elements. Adopting the convention
that the indices $j,k,\ell,\iota=1,\ldots,K$ always denote \emph{student} nodes,
while $n,m=1,\ldots,M$ are reserved for teacher hidden nodes. 

\begin{widetext}
The equation of
motion of the teacher-student density can then be written as
\begin{align}
  \begin{split}
    \label{eq:eom-r}
    \frac{\partial r^{km}(\rho, t)}{\partial t} = -\frac{\eta}{\delta} v^k
    d(\rho) & \left( r^{km}(\rho) \sum_{j\neq k}^K v^j \frac{ Q^{jj}\; I_3(k, k,
        j) -Q^{kj} I_3(k, j, j)} {Q^{jj}Q^{kk}-(Q^{kj})^2}\right.  + \sum_{j\neq
      k}^K v^j r^{j m}(\rho) \frac{ Q^{kk} I_3(k, j,
      j) - Q^{kj}\; I_3(k, k, j) } {Q^{jj}Q^{kk}-(Q^{kj})^2} \\
    &\quad + v^k
    r^{k m}(\rho) \frac{1}{Q^{kk}} I_3(k, k, k) 
    - r^{km}(\rho) \sum_{n}^M \tilde v^n \frac{T^{nn} I_3(k, k, n)
      - R^{kn} I_3(k, n, n) } {Q^{kk}T^{nn}-(R^{kn})^2} \\
    & \quad \left. - \frac{b
        \rho}{d(\rho)} \sum_{n}^M \tilde v^n \tilde T^{nm} \frac{Q^{kk} I_3(k,
        n, n) -R^{kn} I_3(k, k, n)} {Q^{kk}T^{nn}-(R^{kn})^2} \right),
  \end{split}
\end{align}
where $d(\rho)=(c-b^2) \delta + b^2 \rho$. The teacher-teacher overlap
$T^{nm}\equiv\EE \nu^n \nu^m$~\eqref{eq:T}, while $\tilde T^{nm}$ is the overlap
of the teacher weights after rotation into the eigenbasis of $\Omega_{rs}$,
weighted by the eigenvalues $\rho$:
\begin{equation}
  \label{eq:reweighted-tilde-W}
  \tilde T^{mn} \equiv \usR\sum_\tau \rho_\tau \tilde \omega_\tau^m\tilde
  \omega_\tau^n; \qquad  \text{where} \quad \tilde{\omega}_\tau^m=\ussqR \sum_r \tilde w_r^m \psi_{\tau r}.
\end{equation}
In writing the equations, we used the following shorthand for the
three-dimensional Gaussian averages
\begin{equation}
  \label{eq:I3}
  I_3(k, j, n) \equiv \EE\left[ g'(\lambda^k) \lambda^j
    \tilde{g}(\nu^n) \right] ,
\end{equation}
which was introduced by Saad~\& Solla~\cite{Saad1995a}. Arguments passed to
$I_3$ should be translated into local fields on the right-hand side by using the
convention where the indices $j,k,\ell,\iota$ always refer to student local fields
$\lambda^j$, etc., while the indices $n,m$ always refer to teacher local fields
$\nu^n$, $\nu^m$. Similarly,
\begin{equation}
  I_3(k, j, j) \equiv \EE\left[ g'(\lambda^k) \lambda^j g(\lambda^j) \right],
\end{equation}
where having the index $j$ as the third argument means that the third factor is
$g(\lambda^j)$, rather than $\tilde{g}(\nu^m)$ in Eq.~\eqref{eq:I3}. The average
in Eq.~\eqref{eq:I3} is taken over a three-dimensional normal distribution with
mean zero and covariance matrix
\begin{equation}
  \label{eq:Phi3}
  \Phi^{(3)}(k, j, n) = \begin{pmatrix}
    Q^{kk} &  Q^{kj} & R^{kn} \\
    Q^{kj} &  Q^{jj} & R^{jn} \\

    R^{kn} &  R^{jn} & T^{nn}
  \end{pmatrix}.
\end{equation}

For the latent student-student overlap $\Sigma^{k \ell}$, it is again convenient
to introduce the density $\sigma^{k\ell}(\rho, t)$ as
\begin{equation}
  \label{eq:Sigma_int}
  \Sigma^{k\ell}(t)= \;\int \dd \rho\;  p_{\Omega}(\rho)\;  \sigma^{k\ell}(\rho, t),
\end{equation}
whose equation of motion is given by
\begin{align}
  \begin{split}
    \label{eq:eom-sigma}
    \frac{\partial \sigma^{k\ell}(\rho, t)}{\partial t} = -\frac{\eta}{\delta} &
    \left(d(\rho) v^k \sigma^{k\ell}(\rho) \sum_{j\neq k}v^j \frac{Q^{jj} I_3(k,
        k, j) -Q^{kj} I_3(k, j, j)} {Q^{jj}Q^{kk}-(Q^{kj})^2} + v^k \sum_{j\neq
        k} v^j d(\rho) \sigma^{j \ell}(\rho) \frac{ Q^{kk} I_3(k, j, j) - Q^{kj}
        I_3(k, k, j) }
      {Q^{jj}Q^{kk}-(Q^{kj})^2} \right. \\
    &\qquad + d(\rho) v^k \sigma^{k \ell}(\rho)v^k \frac{1}{Q^{kk}} I_3(k, k, k)
    - d(\rho) v^k \sigma^{k\ell}(\rho) \sum_n \tilde v^n\frac{T^{nn} I_3(k, k,
      n) - R^{kn} I_3(k, n, n) }
    {Q^{kk}T^{nn}-(R^{kn})^2} \\
    &\qquad -b \rho v^k \sum_n \tilde v^n r^{\ell n}(\rho) \frac{Q^{kk} I_3(k,
      n, n) -R^{kn}
      I_3(k, k, n)}{Q^{kk}T^{nn}-(R^{kn})^2} \\
    & \qquad + \text{all of the above with } \ell\to k, k\to\ell\Bigg).\\
    & \hspace*{-1em} + \eta^2 v^k v^\ell \left[(c - b^2)\rho +
      \frac{b^2}{\delta}\rho^2\right] \left(
      \sum_{j,\iota}^K v^j v^\iota I_4(k, \ell, j, \iota)\right. \\
    & \hspace*{12.5em} \left. - 2 \sum_j^K \sum_m^M v^j \tilde{v}^m I_4(k, \ell,
      j, m) + \sum_{n,m}^M \tilde{v}^n\tilde{v}^m I_4(k, \ell, n, m) \right).
  \end{split}
\end{align}
This equation involves again the integrals $I_3$ and a four-dimensional average
that we denote
\begin{equation}
  \label{eq:I4}
  I_4(k, \ell, j, n) \equiv \EE\left[ g'(\lambda^k)
    g'(\lambda^\ell) g(\lambda^j) g(\nu^n)\right].
\end{equation}
using the same notational conventions as for $I_3$, so the four-dimensional
covariance matrix reads
\begin{equation}
  \label{eq:Phi4}
  \Phi^{(4)}(k, \ell, j, n) = \begin{pmatrix}
    Q^{kk} &  Q^{k\ell} &  Q^{kj} & R^{kn} \\
    Q^{k \ell} &  Q^{\ell \ell} & Q^{\ell j} & R^{\ell n} \\
    Q^{kj} &  Q^{\ell j} & Q^{jj} & R^{jn} \\

    R^{kn} &  R^{\ell n} & R^{jn} & T^{nn}
  \end{pmatrix}.
\end{equation}

\begin{figure*}[ht!]
  \centering \includegraphics[width=.75\textwidth]{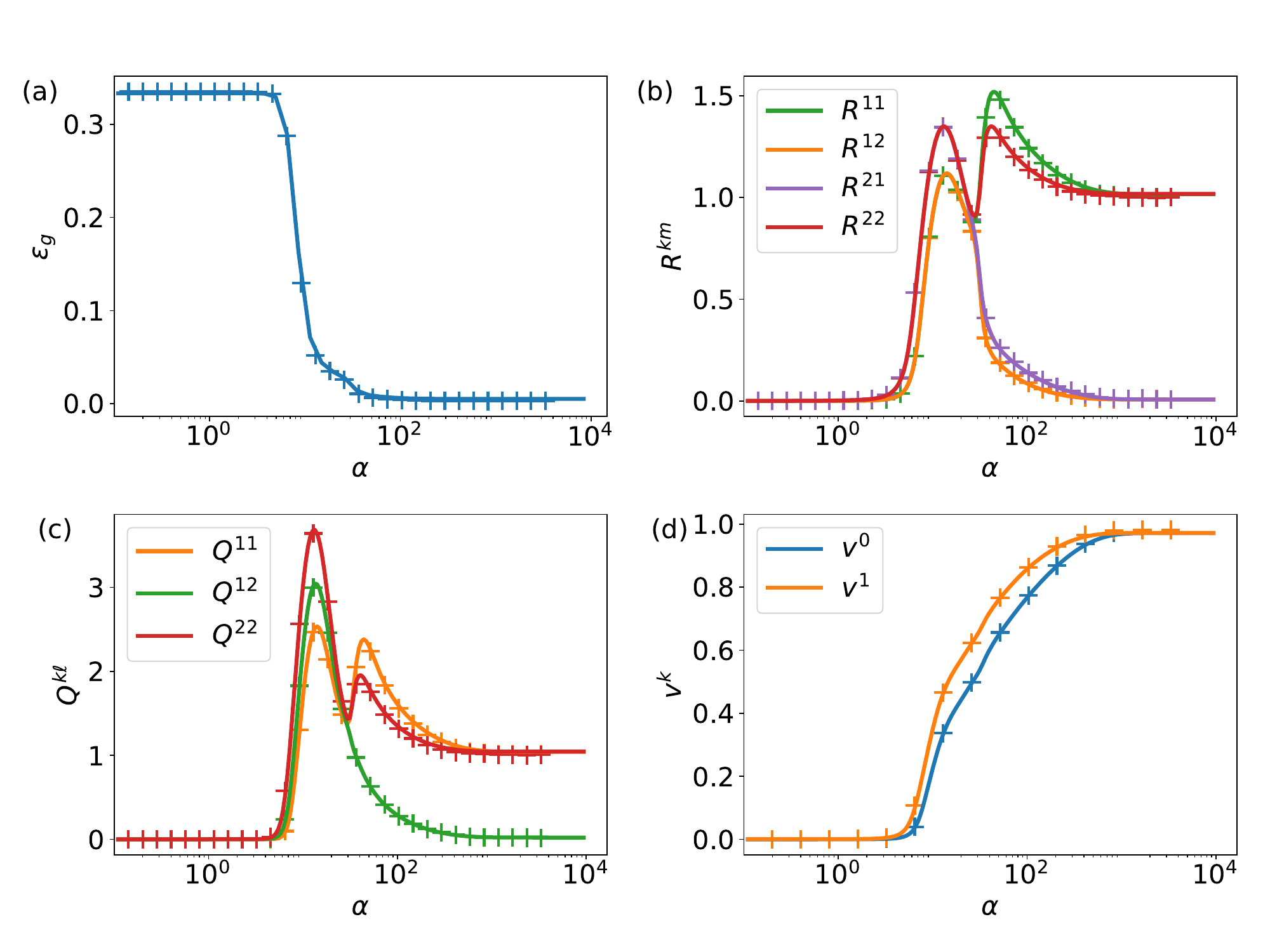}
  \caption{\label{fig:sim_vs_ode} \textbf{The analytical description of the
      hidden manifold generalisation dynamics matches experiments even at
      moderate system size.}  We plot the time evolution of the generalisation
    error $\epsilon_g(\alpha)$ and the order-parameters $R^{km}$, $Q^{k\ell}$
    and 2nd layer weights $v^k$ obtained by integration of the ODEs (solid) and
    from a single run of SGD (crosses).  \emph{Parameters:}
    $g(x)=\erf(x/\sqrt{2}), N=10000, D=100, M=2, K=2, \eta=0.2, \tilde{v}^m=1$.}
\end{figure*}

The equation of motion for the ambient student-student overlap $W^{k\ell}$ can
be written directly:
\begin{align}
  \begin{split}
    \label{eq:eom-W}
    \frac{\dd W^{k\ell}(t)}{\dd t} = &- \eta v^k \left(\sum_{j}^K v^j I_3(k, \ell, j) -
      \sum_n \tilde v^n I_3(k, \ell, n) \right) - \eta v^\ell \left(\sum_{j}^K
      v^j I_3(\ell, k, j) - \sum_n\tilde v^n I_3(\ell, k, n) \right)\\
    & + c\eta^2 v^k v^\ell \left( \sum_{j,a}^K v^j v^a I_4(k, \ell, j, a) - 2 \sum_j^K \sum_m^M
      v^j \tilde v^m I_4(k, \ell, j, m) + \sum_{n,m}^M \tilde v^n \tilde v^m I_4(k, \ell, n, m) \right).
  \end{split}
\end{align}
\end{widetext}
Finally, the ODE for the second-layer weights $v^k$ is straightforwardly given
by
\begin{equation}
  \label{eq:eom-v}
  \frac{\dd v^k}{\dd t} = \eta \left[ \sum_n^M \tilde v_n I_2(k, n) - \sum_j^K
    v^j I_2(k, j) \right],
\end{equation}
where we have introduced the final short-hand
$I_2(k, j) \equiv \EE\left[ g(\lambda^k) g(\lambda^j)\right]$.

\subsubsection{Solving the equations of motion}

The equations of motion are valid for any choice of $f(x)$, $g(x)$ and
$\tilde g(x)$. To solve the equations for a particular setup, one needs to
compute the three constants $a, b, c$~\eqref{eq:abc} and the averages~$I_3$
and~$I_4$~(\ref{eq:I3},~\ref{eq:I4}).
Choosing $g(x)=\tilde g(x)=\erf(x/\sqrt{2})$, they can be computed
analytically~\cite{Biehl1995}. Finally, one needs to determine the spectral
density of the matrix $\Omega_{rs}$. When drawing the entries of the feature
matrix $F_{ir}$ i.i.d.\ from some probability distribution with finite second
moment, the limiting distribution of the eigenvalues $p_\Omega(\rho)$ in the
integral~\eqref{eq:R_int} and~\eqref{eq:Sigma_int} is the Marchenko-Pastur
distribution~\cite{marchenko1967distribution}:
\begin{equation}
  \label{eq:pMP}
  p_{\mathrm{MP}}(\rho)=\frac{1}{2\pi
    \delta}\frac{\sqrt{(\rho_{\max}-\rho)(\rho-\rho_{\min})}}{\rho},
\end{equation}
where $\rho_{\min}=\left(1-\sqrt{\delta}\right)^2$ and
$\rho_{\max}=\left(1+\sqrt{\delta}\right)^2$, where we recall that
$\delta\equiv D/N$. Note that our theory crucially also applies to non-random
matrices; we will visit such an example in Sec.~\ref{sec:hadamard}, where we
also discuss the importance of this use case. A complete numerical
implementation of the equations of motion is available on GitHub~\cite{code}.

We illustrate the content of the equations of motion in Fig.~\ref{fig:sim_vs_ode}, where
we plot the dynamics of the generalisation error, the order parameters $R^{km}$,
$\Sigma^{k \ell}$ and~$W^{k\ell}$ and the second-layer weights $v^k$ obtained from a
single experiment with $N=10000, D=100, M=K=2$, starting from small initial weights
(crosses). The elements of the feature matrix are drawn i.i.d.\ from the standard normal
distribution, as are the elements of the latent representations $\vc$. The solid lines
give the dynamics of these order parameters obtained by integrating the equations of
motion. The initial conditions for the integration of the ODEs were taken from the
simulation. The ODE description matches this single experiment really well even at
moderate system sizes. For Fig.~\ref{fig:sim_vs_ode}, our choice of $N$ and $D$ results
in $\delta=0.01$, and we checked that the ODEs and simulations agree for various values
of $\delta$, \emph{cf.} Fig.~\ref{fig:delta}. 

\subsubsection{Discussion}

\citet{yoshida2019datadependence} recently analysed online learning for
two-layer neural networks~\eqref{eq:phi} trained on Gaussian inputs, with a
two-layer teacher acting directly on the inputs $\vx$. Their approach consists
of introducing distinct order parameters $R^{km}_{(i)}$, $Q^{kl}_{(i)}$ etc.\
for each distinct eigenvalue of the input covariance matrix $\Omega$. They
analysed their equations for covariance matrices with one and two distinct
eigenvalues. Here, we first introduced the GEP~\eqref{gep} to show that inputs
which are not normally distributed, such as
$\mX = f\left( \nicefrac{\mC\mF}{\sqrt{D}} \right)$, can be reduced to an
effective Gaussian model as far as the dynamics of learning are
concerned. Furthermore, the description of the learning dynamics we just
discussed allows us to analyse inputs with any well-defined spectral density
with just a single set of order parameters $Q^{kl}$, $R^{km}$ and $T^{nm}$. This
is made possible by introducing the integral over the order parameter densities
$r^{km}(\rho)$ etc. As we will see below, this integral can actually be solved
for small $\delta$, which simplifies the equations of motion considerably and
allowing for a detailed analysis (\emph{cf.}  Sec.~\ref{sec:small-delta}).

We lastly comment on the role of the dimensionality in our setup. Inspection of
the test error~\eqref{eq:eg_preactivations} reveals that a student has to
recover the local fields of the teacher $\nu^m$ in order to perform well (if she
has the same activation function as the teacher). If the student was trained
directly on the latent variables $\mC$, she could recover these local fields
perfectly and we would be back in the setup of~\citet{Saad1995a}. In the HMM,
the student is only given the high-dimensional inputs $\mX$, which can be seen
as a noisy projection of the latent variables $\mC$~\eqref{eq:tildeX}. The high
dimensionality of the student inputs is thus a constraint that must be overcome
to learn well, because projection to high dimensions is part of the
\emph{data-generating} process. This is to be contrasted with setups like random
features~\cite{rahimi2008random, rahimi2009weighted} or certain neural circuits
in sensory processing~\cite{babadi2014sparseness, kadmon2016optimal}, where
projection of the inputs to higher-dimensional spaces is part of the
\emph{analysis} and generally simplifies the subsequent learning problem.

\section{Analytical Results}
\label{sec:results}

The goal of this section is to use the analytic description of online learning
to analyse the dynamics and the performance of two-layer neural networks in
detail.

\subsection{Specialisation of student nodes in the HMM}
\label{sec:specialisation}

An intriguing feature of both the canonical teacher-student setup and the hidden
manifold model is that they both exhibit a \emph{specialisation} phenomenon. Upon closer
inspection of the time evolution of the order parameter $R^{km}$ in
Fig.~\ref{fig:sim_vs_ode} (b), we see that during the initial decay of the
generalisation error up to a time $t=\nicefrac{P}{N}\sim10$, all elements of the matrix
$R^{km}$ are comparable.
In other words, the correlations between the pre-activation
$\lambda^k$ of any student node and the pre-activation $\nu^m$ of any teacher
node is roughly the same. As training continues, the student nodes
``specialise'': the pre-activation of one student node becomes strongly
correlated with the pre-activation of only a single teacher node. In the example
shown in Fig.~\ref{fig:sim_vs_ode}, we have strong correlations between the
pre-activation of the first student and the first teacher node ($R^{11}$), and
similarly between the second student and second teacher node ($R^{22}$). The
specialisation of the teacher-student correlations is concurrent to a
de-correlation of the student units, as can be seen from the decay
of the off-diagonal elements of the latent and ambient student-student overlaps
$\Sigma^{k\ell}$ and $W^{k\ell}$, respectively (bottom of
Fig.~\ref{fig:sim_vs_ode}). Similar specialisation transitions have been
observed in the canonical teacher-student setup for both online and batch
learning~\cite{Riegler1995, Saad1995a}; see~\citet{Engel2001} for a
review.

\subsection{Using non-random feature matrices}
\label{sec:hadamard}

Our first example of the learning dynamics in Sec.~\ref{sec:specialisation} was
for a feature matrix $\mF$ whose entries were taken i.i.d.\ from the normal
distribution. The derivation of the ODEs for online learning however does not
require that the feature matrix $\mF$ be random; instead, it only requires the
balance condition stated in Eq.~\eqref{eq:balance} as well as the normalisation
conditions~\eqref{eq:F-cond}. To illustrate this point, we plot examples of
online learning dynamics with $M=K=2$ in Fig.~\ref{fig:sim_vs_ode_hadamard},
with the prediction from the ODE as solid lines and the result of a single
simulation with crosses. In blue, we show results where the elements of $F_{ir}$
were drawn i.i.d.\ from the standard normal distribution. For the experiment in
orange, $\mF=\mH_N$, where $\mH_N$ is a Hadamard
matrix~\cite{hadamard1893resolution}. Hadamard matrices are $N\times N$
matrices, hence $\delta=1$, and are popular in error-correcting codes such as
the Reed-Muller code~\cite{muller1954application, reed1953class}. They can be
defined via the relation
\begin{equation}
  \mH_N \mH_N^\top = N \mathbb{I}_N,
\end{equation}
where $\mathbb{I}_N$ is the $N\times N$ identity matrix. As we can see from
Fig.~\ref{fig:sim_vs_ode_hadamard}, the ODEs capture the generalisation dynamics
of the Hadamard-case just as well.

\begin{figure}[t]
  \centering
  \includegraphics[width=\linewidth]{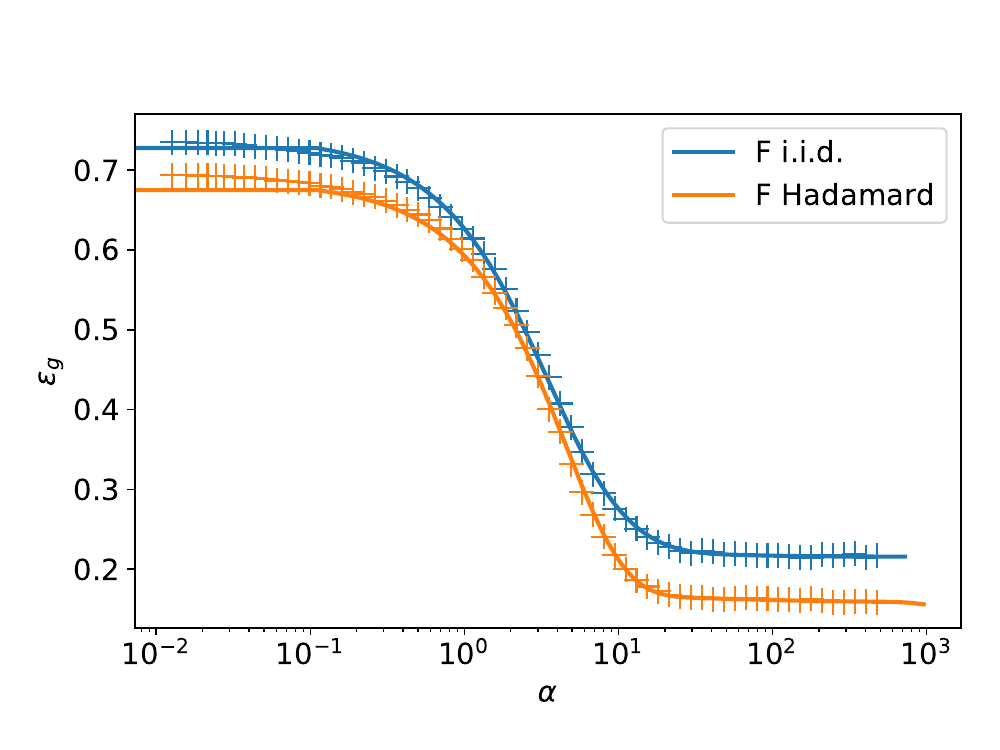}
  \caption{\label{fig:sim_vs_ode_hadamard}\textbf{The ODE analysis is
      asymptotically correct for non-random feature matrices $\mF$.}  We plot
    the time evolution of the generalisation error $\epsilon_g$ obtained by
    integration of the ODEs (solid) and from a single run of
    SGD~\eqref{eq:sgdtheta} (crosses) for two different matrices $\mF$: (i)
    elements $F_{ir}$ are drawn i.i.d.\ from the standard normal distribution
    (blue); (ii) $\mF$ is a Hadamard matrix~\cite{hadamard1893resolution}
    $f(x)=\mathrm{sgn}(x), g(x)=\tilde{g}(x)=\erf(x / \sqrt{2}), N=1023, D=1023,
    M=2, K=2, \eta=0.2, \tilde{v}^m=1$.}
\end{figure}

\subsection{The limit of small latent dimension}
\label{sec:small-delta}

The key technical challenge in analysing the analytical description of the
dynamics is handling the integro-differential nature of the equations. 
We can simplify the equations in the limit of small $\delta\equiv D /
N$. Numerical integration of the equations reveals that at convergence, the
continuous order parameter densities $r^{km}(\rho)$ and $\sigma^{k \ell}(\rho)$
are approximately constant:
\begin{equation}
  \label{eq:constant-r}
  r^{km}(\rho) = r^{km}; \quad \sigma^{k \ell}(\rho) = \sigma^{k \ell}
\end{equation}
This is a key observation, because making the ansatz~\eqref{eq:constant-r}
allows us to transform the integro-differential equations for the dynamics of
$r^{km}(\rho, t)$~\eqref{eq:eom-r} and
$\sigma^{k\ell}(\rho, t)$~\eqref{eq:eom-sigma} into first-order ODEs, provided
we can perform the integral over the eigenvalue distribution $p_\Omega(\rho)$ in
Eqs.~\eqref{eq:R_int} and~\eqref{eq:Sigma_int} analytically.  This is for
example the case if we take the elements of the feature matrix $\mF$ i.i.d.\
from any probability distribution with bounded second moment, in which case
$p_\Omega(\rho)$ is given by the Marchenko-Pastur
distribution~\eqref{eq:pMP}. We will focus on this case for the remainder of
this section.

Let us note that the regime of small delta is also the relevant regime for image
data sets such as MNIST and CIFAR10, whose $\delta$ has been estimated
previously to be around $\delta_{\mathrm{MNIST}}\sim ~14/784$ and
$\delta_{\mathrm{CIFAR10}}\sim 35 / 3072$, respectively~\cite{Grassberger1983,
  Costa2004, Levina2004a, Spigler2019}; \emph{cf.} our discussion in the
Introduction.

\subsubsection{The effect of the latent dimension $D=\delta N$}
\label{sec:eg-vs-delta}

\begin{figure}
  \includegraphics[width=\linewidth]{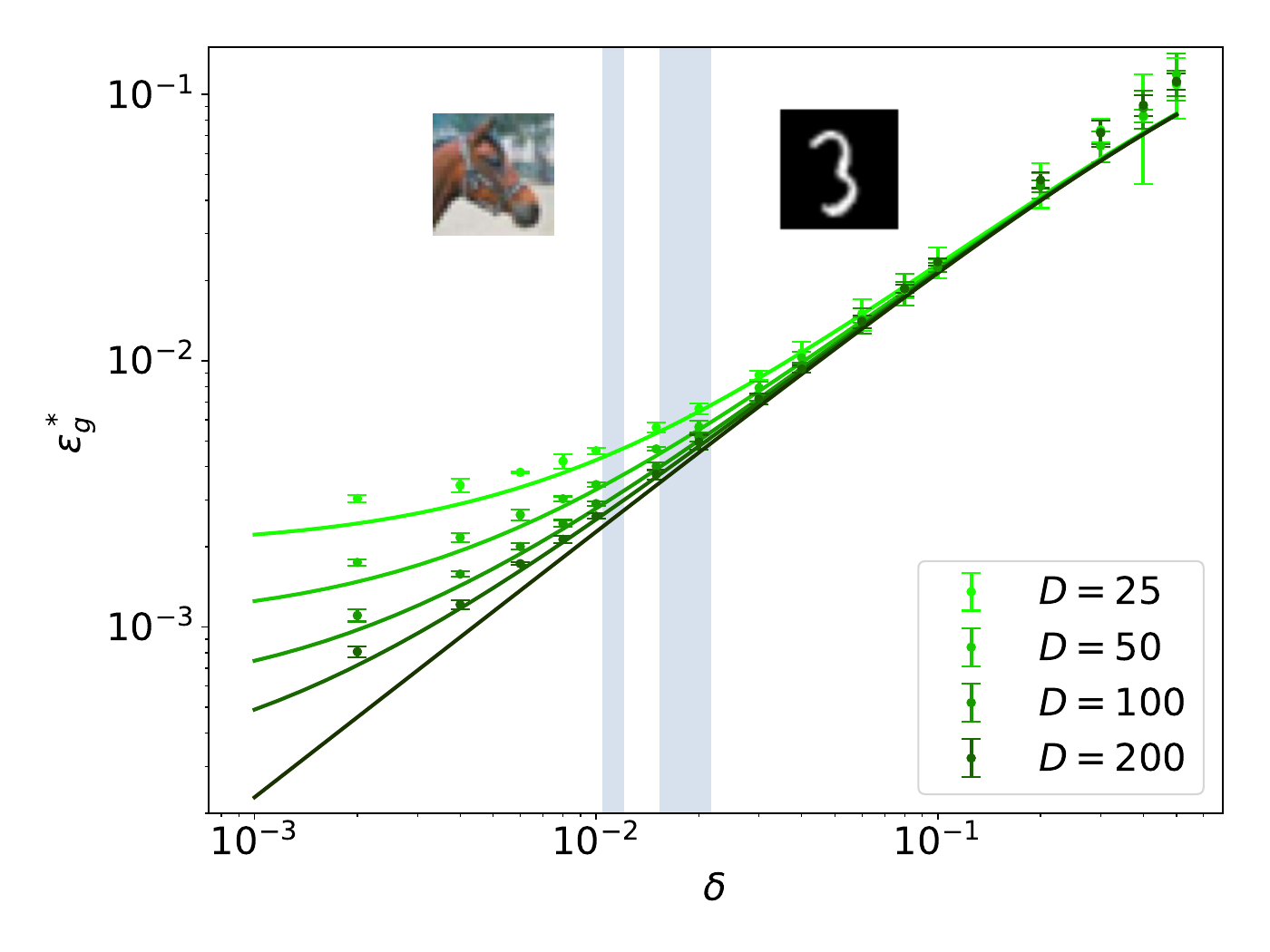}
  \caption{\label{fig:delta} \textbf{The impact of the latent dimension
      $\delta\equiv D/N$}. We plot the final test error $\epsilon_g^*$ of
    sigmoidal students trained on the hidden manifold model with three different
    intrinsic dimensions $D$ as a function of $\delta=D/N$, where $N$ is the
    input dimension. The average is taken over five runs. The solid lines are
    the asymptomatic theoretical predictions derived in
    Sec.~\ref{sec:eg-vs-delta}. The shaded bars indicate experimental
    estimates~\cite{Grassberger1983, Costa2004, Levina2004a, Spigler2019} for
    $\delta$ for the CIFAR10 data set (left) and the MNIST data set (right).
    $f(x)=\mathrm{sgn}(x), g(x)=\tilde{g}(x)=\erf(x/\sqrt{2}), M=K=2, \eta=0.2,
    \tilde v^m=1$}
\end{figure}

As a first application of this approach, we analyse the dependence of the
asymptotic test error $\epsilon_g^*$ on the latent dimension $D$ of the hidden
manifold when teacher and student have the same number of hidden nodes, $K=M$.

From inspection of the form of the order parameters after integrating the full
set of ODEs until convergence, we made the following ansatz for the overlap
matrices:
\begin{align}
    \label{eq:ansatz_matched}
          \Sigma^{k\ell} &= \begin{cases}
            S & \quad k=\ell, \\
            s & \quad \mathrm{otherwise},
          \end{cases}\quad
    W^{k\ell} &= \begin{cases}
      W & \quad k=\ell, \\
      w & \quad \mathrm{otherwise},
    \end{cases}\\
          T^{nm} &= \begin{cases}
            T & \quad n=m, \\
            t & \quad \mathrm{otherwise},
          \end{cases}\quad
    \tilde{T}^{nm} &= \begin{cases}
      \tilde{T} & \quad n=m, \\
      \tilde{t} & \quad \mathrm{otherwise},
    \end{cases} \\
      R^{km} &= \begin{cases}
      R & \quad k=m, \\
      r & \quad \mathrm{otherwise}
    \end{cases} 
    & v^k = v; \quad A^m = A
\end{align}
Substituting this ansatz into the ODEs allowed us to derive closed-form
expressions for ODEs governing the dynamics of seven order parameters
$R, r, S, s, W, w$ and $v$ that are valid for small $\delta$ and for any
$K=M$. The teacher-related order parameters $T, t, \tilde{T}$ and $\tilde{t}$
describe the teacher and are constants of the motion. They have to be chosen to
reflect the distribution from which the weights of the teacher network are drawn
in an experiment. The full equations of motion are rather long, so instead of
printing them here in full we provide a Mathematica notebook for
reference~\cite{code}.

The key idea of our analytical approach is to look for fixed points of this ODE
system and to substitute the values of the order parameters at those fixed
points into the expression for the generalisation
error~\eqref{eq:eg-order-parameters}. To understand the structure of the fixed
points of the ODEs, we ran a numerical fixed point search of the ODEs from 1000
initial values for the order parameters drawn randomly from the uniform
distribution. We found two types of solution. First, there exist solutions of
the form $R=r$, $S=s$ and $W=w$. This solution is a saddle point of the
equations and is thus not a stable fixed point of the dynamics. Instead, it
corresponds to a well-known ``unspecialised'' phase, when networks with $K>1$
hidden nodes have not yet specialised and hence achieve only the performance of
a network with $K=1$ hidden unit (cf.\ our discussion in
Sec.~\ref{sec:specialisation}). The learning dynamics approaches this saddle
point at an intermediate stage of learning, but finally drifts away from it
towards a ``specialised'' solution. This second solution corresponds to the
asymptotic fixed point of the learning dynamics where the student has
specialised, \emph{i.e.} we have $R$ large and $r$ small, etc. Substituting the
values of the order parameters of this solution into
Eq.~\eqref{eq:eg-order-parameters} yields the asymptotic generalisation error of
a student.

Making this argument rigorous would require a proof of global convergence of the
coupled, non-linear integro-differential equations of motion
(\ref{eq:eom-r},~\ref{eq:eom-sigma},~\ref{eq:eom-W},~\ref{eq:eom-v}) from random
initial conditions. This is a challenging mathematical problem that remains
open, despite some recent progress for two-layer neural networks with finite $N$
and large hidden layer~\cite{Soltanolkotabi2018, Mei2018, Rotskoff2018,
  Chizat2018, Sirignano2018}. Thus all predictions in this way ultimately need
to be compared to simulations to verify their accuracy.

We show the results of this analysis in Fig.~\ref{fig:delta}. The crosses are
experimental results for which we trained networks with $M=K=2$ on data from a
hidden manifold with latent dimension $D=25, 50, 100$ and 200, choosing the
input dimension $N$ to obtain the range of $\delta$ desired for each curve. We
plot the asymptotic error averaged over five runs with dots; error bars
indicated two standard deviations. The lowest solid line in Fig.~\ref{fig:delta}
is the theoretical prediction obtained by the procedure just explained when
assuming that $T=1, t=0, \tilde T = 1, \tilde t =0$.

While the experimental results are approaching the theoretical line as the
latent dimension $D$ increases, there are qualitative differences in the shape
of the $\delta$ dependence for small $\delta$. These differences arise due to
the following finite-size effect. While it is numerically easy to enforce
$T=1, t=0$ by orthogonalising the teacher weight matrix, it is not possible to
explicitly control the re-weighted teacher-teacher overlap
$\tilde T^{nm}$~\eqref{eq:reweighted-tilde-W}. The deviation of $\tilde T^{nm}$
from the identity lead to the deviations we see at small $\delta$. We
demonstrate this in Fig.~\ref{fig:delta} by also plotting theoretical
predictions for $\tilde T = 1-x, \tilde t=x$ and choosing $x=1/D$. These curves
match the experiments much better. Plotting the data with a linear y-scale (not
shown) reveals that the solution obtained making the small-$\delta$
ansatz~\eqref{eq:constant-r} is valid until $\delta\sim0.2$.

\subsubsection{Learning rate $\eta$}

\begin{figure}
  \includegraphics[width=\linewidth]{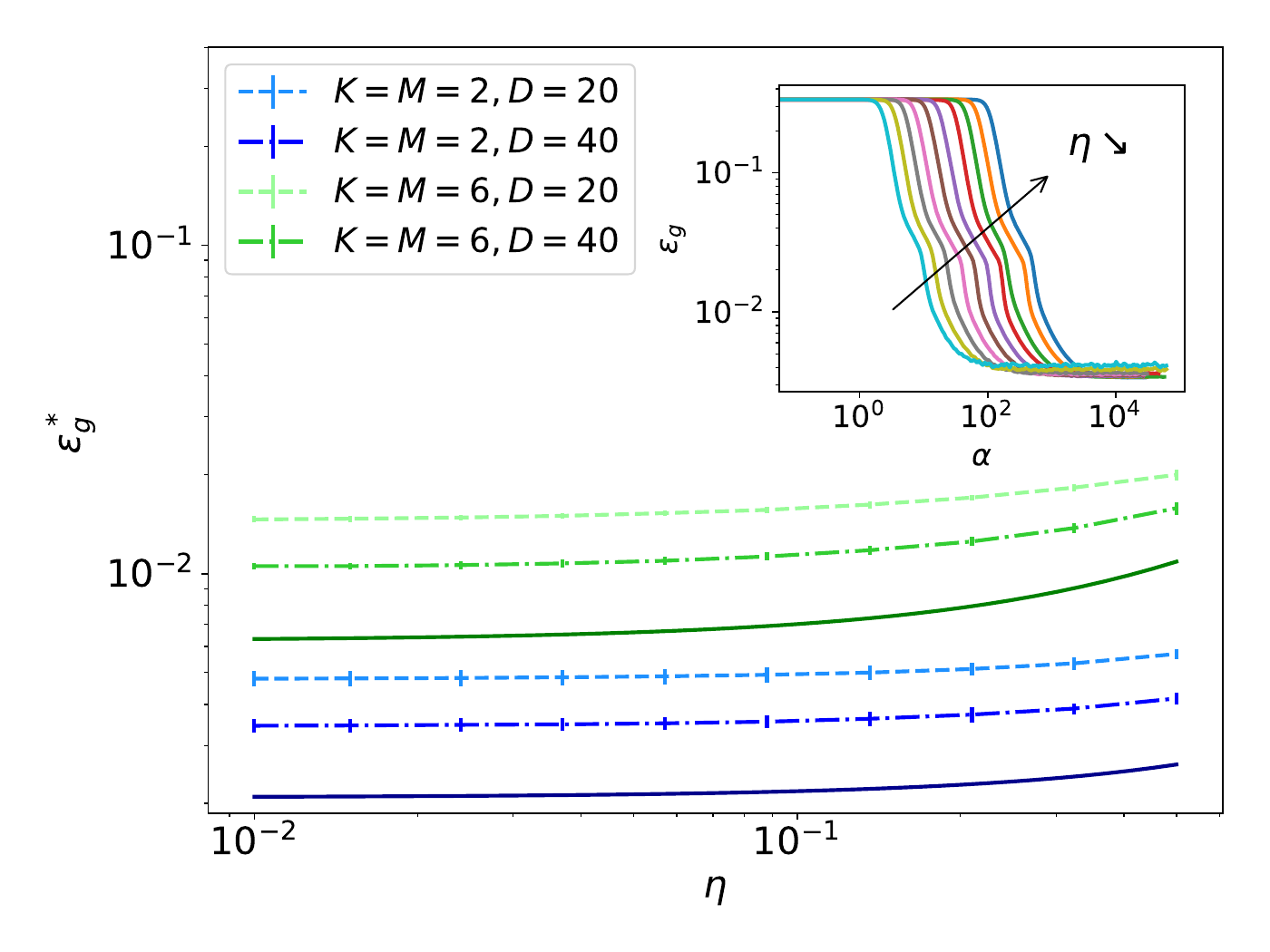}
  \caption{\label{fig:lr} \textbf{The impact of the learning rate $\eta$}. We
    plot the final test error $\epsilon_g^*$ of sigmoidal students trained on
    the hidden manifold model for a range of learning rates $\eta$ for sigmoidal
    networks with $K=M=2$ (blue) and $K=M=6$ (green). We repeated the
    experiments for two values of $D$, choosing $N$ such that
    $\delta=D/N=0.01$. \emph{(Inset)} Generalisation dynamics during training
    ($K=M=2$).  \emph{Parameters:}
    $f(x)=\mathrm{sgn}(x), g(x)=\tilde{g}(x)=\erf(x/\sqrt{2}), \delta=0.01,
    \tilde v^m= 1, K=M$.}
\end{figure}

\begin{figure*}
  \centering
  \includegraphics[width=.45\linewidth]{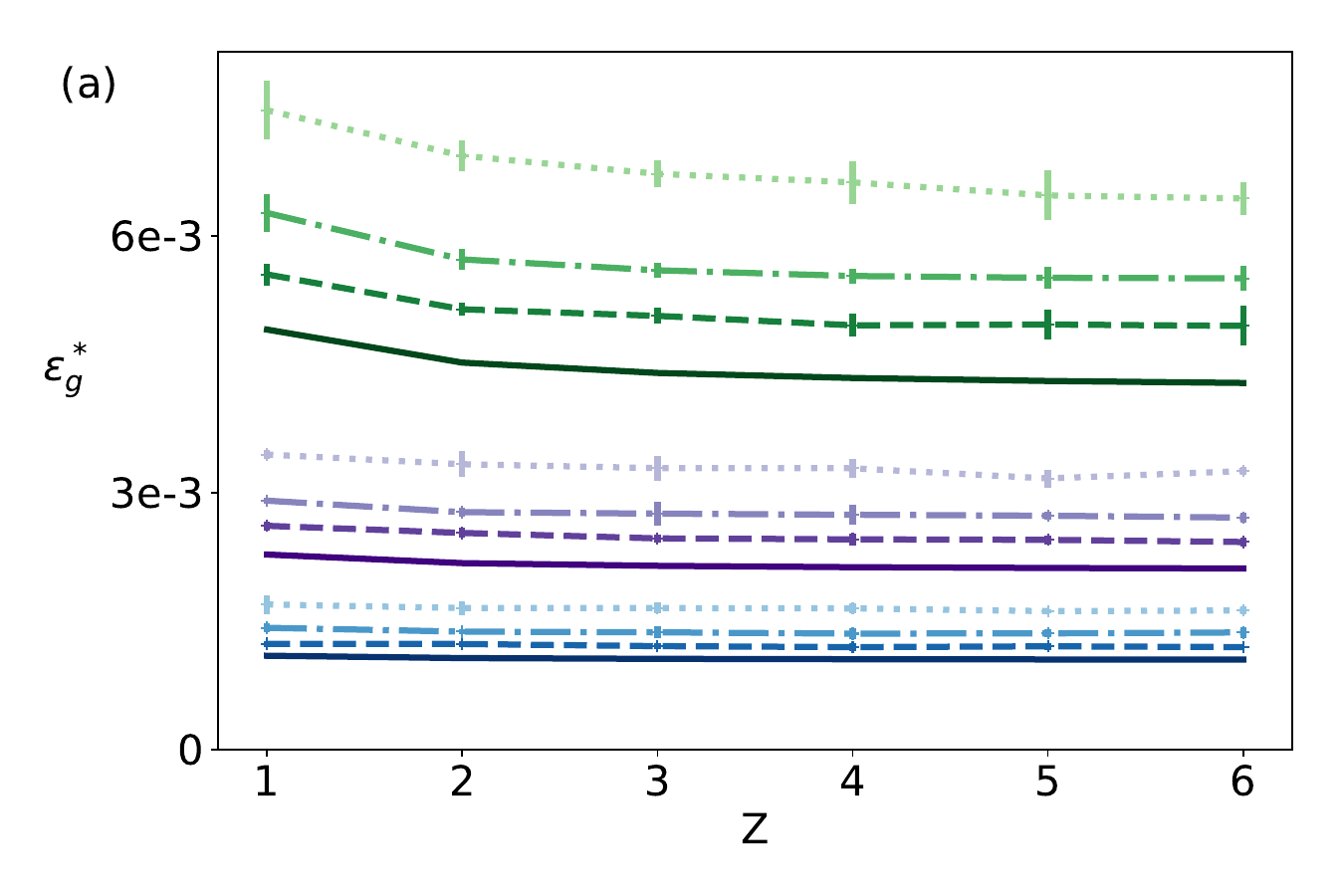}%
  \includegraphics[width=.25\linewidth]{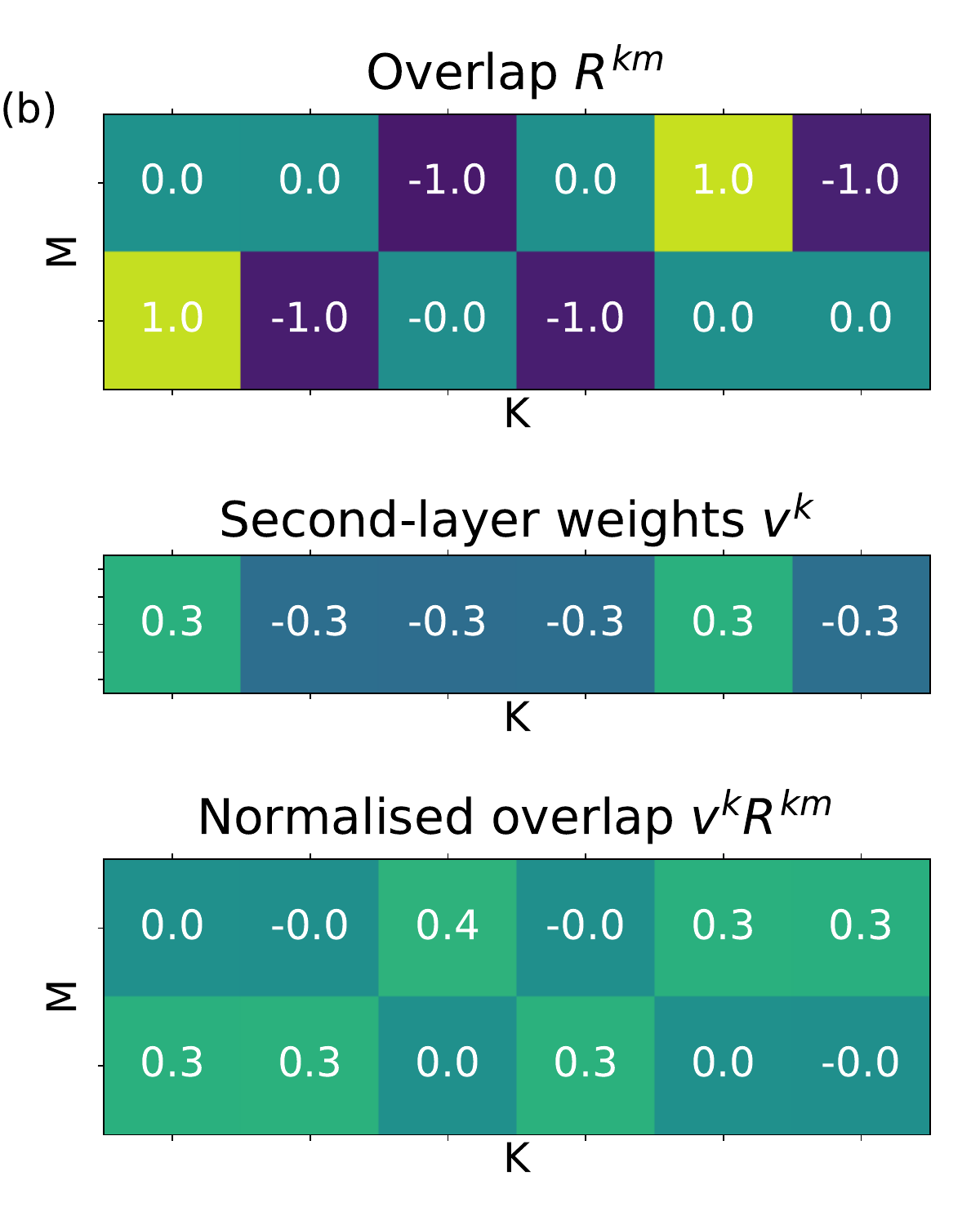}
  \includegraphics[width=.25\linewidth]{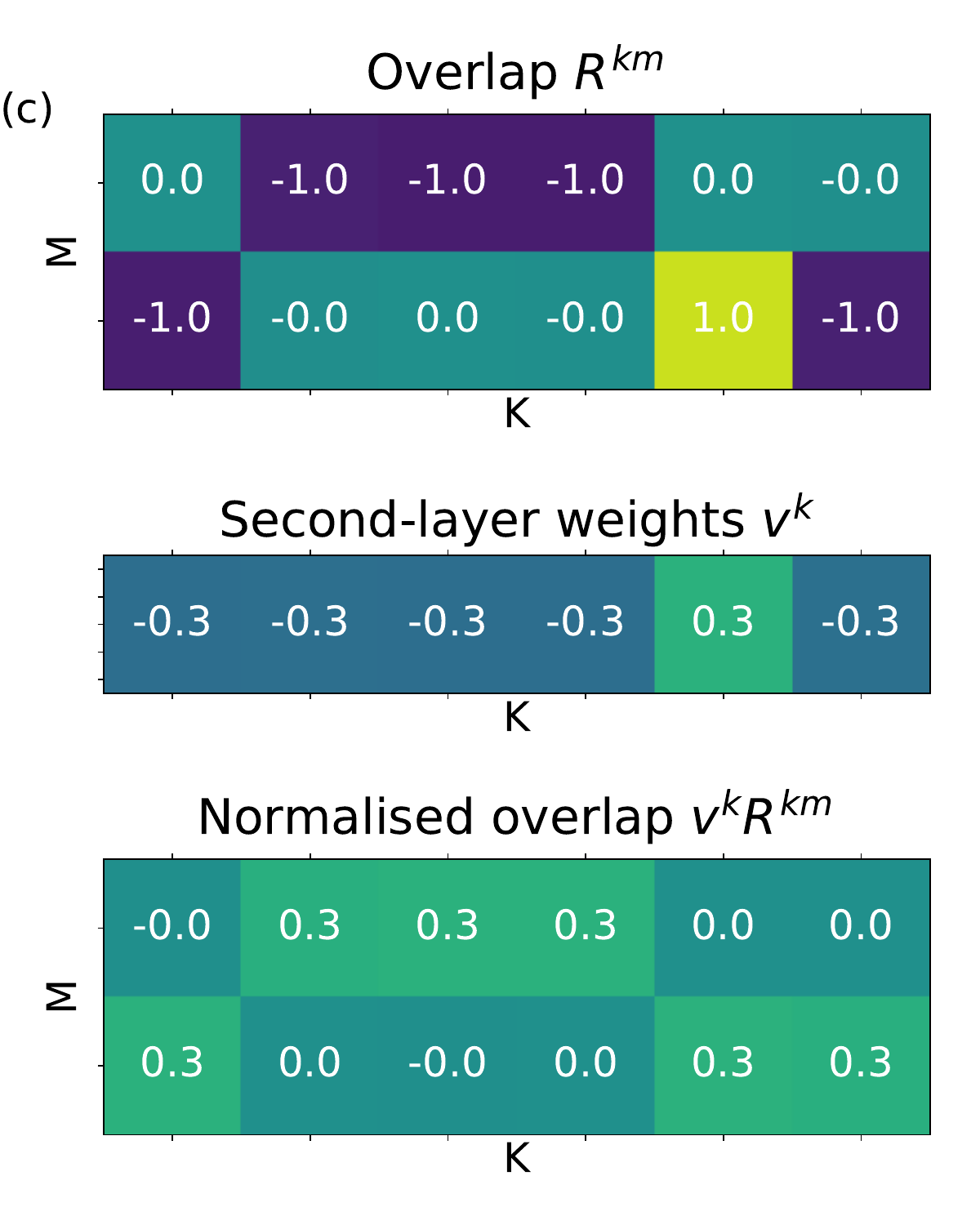}
  \caption{\label{fig:eg-vs-K} \textbf{(a)} \textbf{Asymptotic generalisation
      for online learning} of a student with $K=ZM$ hidden nodes learning from a
    teacher with $M=1$ (blue), $M=2$ (violet) and $M=4$ (green) hidden nodes,
    respectively. The dotted, dashed-dotted and dashed lines correspond to
    $D=50, 100$ and $200$, respectively. Error bars indicated two standard
    deviations over five runs. The solid line is the theoretical prediction
    obtained for $\tilde t=0$. \textbf{(b, c)} Teacher-student overlap
    $R^{km}$~\eqref{eq:R}, second-layer weights $v^k$ and the normalised overlap
    $v^k R^{km}$ obtained in two simulations used in the left plot with
    $M=2, K=6$, starting from different initial conditions, all other things
    being equal. \emph{Parameters:} In all plots,
    $f(x)=\mathrm{sgn}(x), g(x)=\erf(x/\sqrt{2}), \eta=0.2, \tilde v^m=1,
    \delta=0.01, \eta=0.2$. }
\end{figure*}

We found that the asymptotic test error $\epsilon_g^*$ depends only weakly on
the learning rate $\eta$, as we show in Fig.~\ref{fig:lr} for $M=K=2$ and
$M=K=6$, together with the theoretical prediction for $\tilde t=0$. This
theoretical prediction is again obtained by using the
ansatz~\eqref{eq:ansatz_matched} for the order parameters and solving the
resulting fixed point equations, as described in the previous section, but this
time varying the learning rate $\eta$. The weak dependence of $\epsilon_g$ on
$\eta$ should be contrasted with the behaviour the canonical teacher-student
setup, where the generalisation error is proportional to the learning rate in
the case of additive Gaussian output noise~\cite{Saad1997,Goldt2019b}.

In the inset of Fig.~\ref{fig:lr}, we plot the generalisation dynamics of a
neural network trained on the HMM at different learning rates. As expected, the
learning rate controls the speed of learning, with increased learning rates
leading to faster learning until the learning rate becomes so large that
learning is not possible anymore; instead, the weights just grow to infinity.

\subsubsection{The impact of student size}
\label{sec:eg-vs-K}

Another key question in our model is how the performance of the student depends
on her number of nodes~$K$. Adding hidden units to a student who has less hidden
units than her teacher ($K<M$) improves her performance, as would be
expected. This can be understood in terms of the specialisation discussed in
Sec.~\ref{sec:specialisation}: each additional hidden node of the student
specialises to another node of the teacher, leading to improved performance. We
will see an example of this below in Sec.~\ref{sec:complexity}.

But what happens if we give the student more nodes than her teacher has $K>M$?
It is instructive to first study the overlap matrices at the end of training. We
show two examples from an experiment with $M=2, K=6$ at $\delta=0.01$ for
networks starting from different initial conditions. In particular, we plot the
rescaled teacher-student overlap matrix $v^k R^{km}$ in Fig.~\ref{fig:eg-vs-K}
(b, c). We rescale $R^{km}$ by the second-layer weights to account for two
effects: first, the relative influence of a given node to the output of the
student, which is determined by the magnitude of the corresponding second-layer
weight; and second, we have a symmetry in the output of the student since for
sigmoidal activation function,
$v^k g(\vw^k \vx/\sqrt{N}) = -v^k g(-\vw^k \vx / \sqrt{N})$.

\begin{figure*}
  \centering
  \includegraphics[width=.9\linewidth]{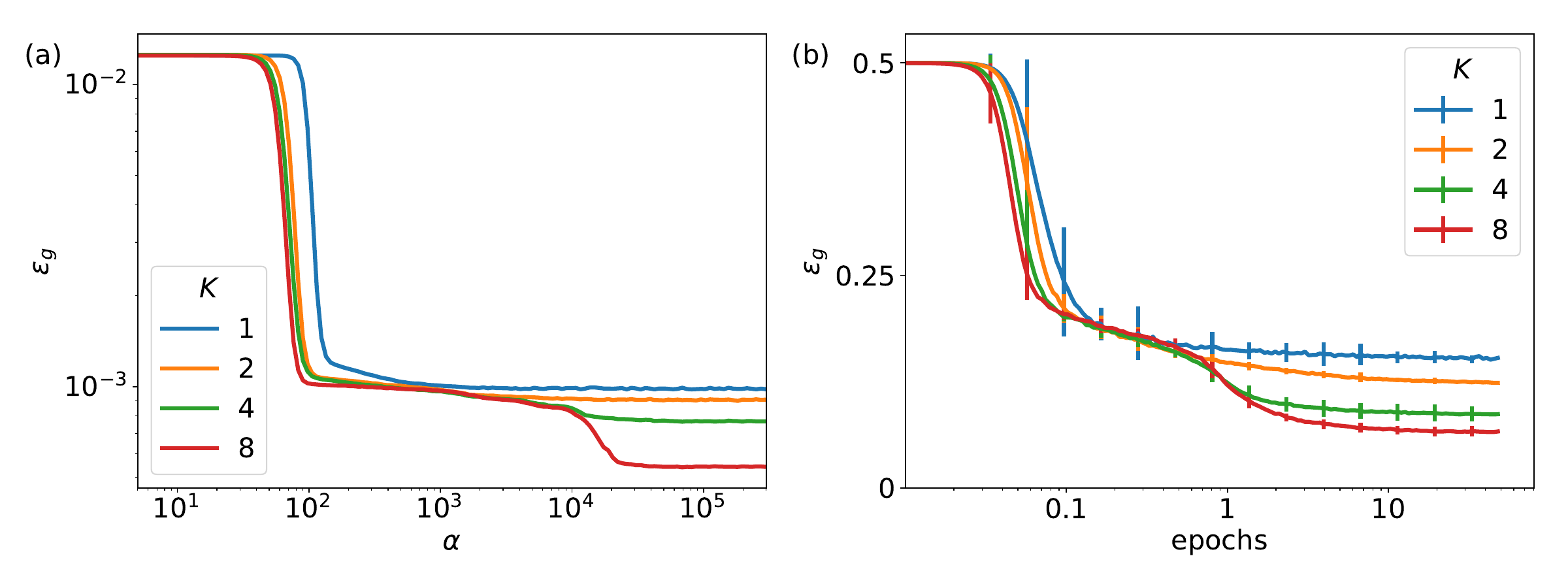}
  \caption{\label{fig:inc_comp} \textbf{Two-layer neural networks learn
      functions of increasing complexity.} We plot the generalisation error of
    sigmoidal two-layer networks with increasing number of hidden nodes $K$
    during a single run of online learning with the hidden manifold model with
    $\delta=0.05, D=25, M=10, \tilde v^m=\nicefrac{1}{M}$ on the HMM \emph{(a)} and
    when trained on odd-versus-even digit classification on MNIST, averaged over
    ten runs \emph{(b)}. Error bars indicate two standard deviations. For
    details, see Sec.~\ref{sec:complexity} and~\ref{sec:supp_complexity}.
    $g(x)=\erf\left(x/\sqrt{2}\right), \eta=0.2, N=784$. \emph{(b)}: batch size
    32.}
\end{figure*}

In the two overlap plots for $K>M$ in Fig.~\ref{fig:eg-vs-K}, the student nodes
display \emph{many-to-one} specialisation: several hidden units of the student
specialise to the same hidden node of the teacher, essentially providing several
estimates of the value of this teacher node. Note that each student node
specialises to one and only one of the teacher nodes, rather than a combination
of two or more teacher nodes. We found this pattern of activations consistently
across all of our runs for various $K$ and $M$. The fact that student nodes are
evenly distributed across teacher nodes is further motivated by the fact that
such an arrangement minimises the generalisation error if the second-layer
teacher weights $\tilde{v}^m$ have equal magnitude, and its first-layer weights
$\tilde \vw^m$ have the same norm. We anticipate that this specialisation pattern is
at least in part due to the sigmoidal form of the activation function~$g(x)$. We
note that the same many-to-one specialisation of hidden units has been
previously reported for the same two-layer networks trained on i.i.d.\
inputs~\cite{Goldt2019b}, and that a similar pattern of specialisation was
observed for networks with finite input and wide hidden layer, where this type
of specialisation was referred to as ``distributional dynamics''~\cite{Mei2018,
  Rotskoff2018, Chizat2018, Sirignano2018}.

These observations motivate the following ansatz for the overlaps of a student with
$K=ZM$ hidden nodes~($Z\in\mathbb{N}$)
\begin{align}
    \label{eq:ansatz_Z}
    R^{km} &= \begin{cases}
      R & \quad k \mod M = m \mod M, \\
      r & \quad \mathrm{otherwise}
    \end{cases}\\
    \Sigma^{k\ell} &= \begin{cases}
      S & \quad k \mod M = \ell \mod M, \\
      s & \quad \mathrm{otherwise}
    \end{cases}
\end{align}
and similarly for $W^{k\ell}$, while we use the same parameterisation for the
teacher order parameters $T, t, \tilde T, \tilde t, A$ and $v$. Searching again
for specialised fixed points of the resulting equations for the seven
time-dependent order parameters $R, r, S, s, W, w$ and $v$ and substituting
their values into Eq.~\eqref{eq:eg-order-parameters} yields the predictions we
indicate by solid lines in Fig.~\ref{fig:eg-vs-K}, where we plot the asymptotic
test error as a function of $Z\equiv K / M$.  We can see small performance
improvements as the student size increases. We also plot, for the three values
of $M$ used, the asymptotic test error measured in experiments with $D=50, 100$
and $200$. As we increase $D$, the experimental results approach the theoretical
prediction for $D\to\infty$.

We finally note that fixed points of the online dynamics with many-to-one specialisation
have been described previously in the canonical teacher-student setup~\cite{Goldt2019b},
who found that this behaviour leads to a more significant improvement of student
performance as $K$ increases for teacher tasks with $y^* = \phi(\vx)$ compared to the
improvement we observe for the HMM. The same type of many-to-one specialisation was also
found by \citet{Mei2018} and \citet{Chizat2018}, who considered a complementary regime
where the input dimension $N$ stays finite while the size of the hidden layer goes to
infinity.

\section{Comparing the hidden manifold model to real data}
\label{sec:comparison}

We finally turn our attention to the comparison of the hidden manifold model to
more realistic data sets, in our cases classic image databases such as CIFAR10
(see Fig.~\ref{fig:intuition} for two examples of images in CIFAR10).

\subsection{Neural networks learn functions of increasing complexity}
\label{sec:complexity}

The specialisation transition that we discussed in Sec.~\ref{sec:specialisation}
has an important consequence for the performance of the neural network, as we
show in Fig.~\ref{fig:inc_comp}. As we train increasingly large student networks
on a teacher with $M=10$ hidden units and second-layer weights
$\tilde v^m=\nicefrac{1}{M}$, we observe that learning proceeds in two
phases. First, there is an initial decay of the generalisation error
until all students have roughly the same test error as the student with a single
hidden unit $K=1$. In a second phase, students with $K>1$ break away from this
plateau after further training and achieve superior performance, with the larger
networks performing better. These improvements are a result of specialisation
after~$\sim 10^3$ epochs, which permits the student network to capitalise on
their additional hidden nodes.

This way of visualising specialisation not only illustrates its importance for
student performance, it is also applicable when training the same two-layer
neural networks on more realistic data sets such as MNIST
(Fig.~\ref{fig:inc_comp}~b) or Fashion MNIST~\cite{xiao2017online} and CIFAR
(Fig.~\ref{fig:supp_inc_comp_sigmoidal}). The plots demonstrate clearly that in
all these cases, the larger networks proceed by first learning functions that
are equivalent to the smaller networks.

In all cases, specialisation is preceded by a plateau where the generalisation
error stays constant because the student is stuck at a saddle point in its
optimisation landscape, corresponding to the unspecialised solution. This
plateau has been discussed extensively in the canonical teacher-student
setup~\cite{Saad1995b, Biehl1996, Rattray1998, Engel2001} and more recently in
the context of recurrent and deep neural networks~\cite{Saxe2014,
  dauphin2014identifying}. By comparing students of different sizes, this
plateau can also be demonstrated on image data sets, as we have done above. This
learning of functions with increasing complexity has also been observed in deep
convolutional networks by~\citet{kalimeris2019sgd}, who used quantities from
information theory to quantify how well one model explains the performance of
another.

These observations are interesting because they suggest how to explain the
ability of neural networks to generalise well from examples when they have
\emph{many more} parameters than samples in their training data set. This is a
key open problem in the theory of deep learning, since the intuition from
classical statistics suggests that in these cases, the networks overfit the
training data and thus generalise poorly~\cite{vapnik1998statistical,
  Mohri2012}. 
It is possible that by learning functions of increasing complexity, networks are
biased towards simple classifiers and avoid over-fitting if their training is
stopped before convergence. This topic is an active research
area~\cite{farnia2018spectral, rahaman2018spectral}.

\subsection{Memorisation of random and realistic data}
\label{sec:memorisation}

\begin{figure*}
  \includegraphics[width=0.5\textwidth]{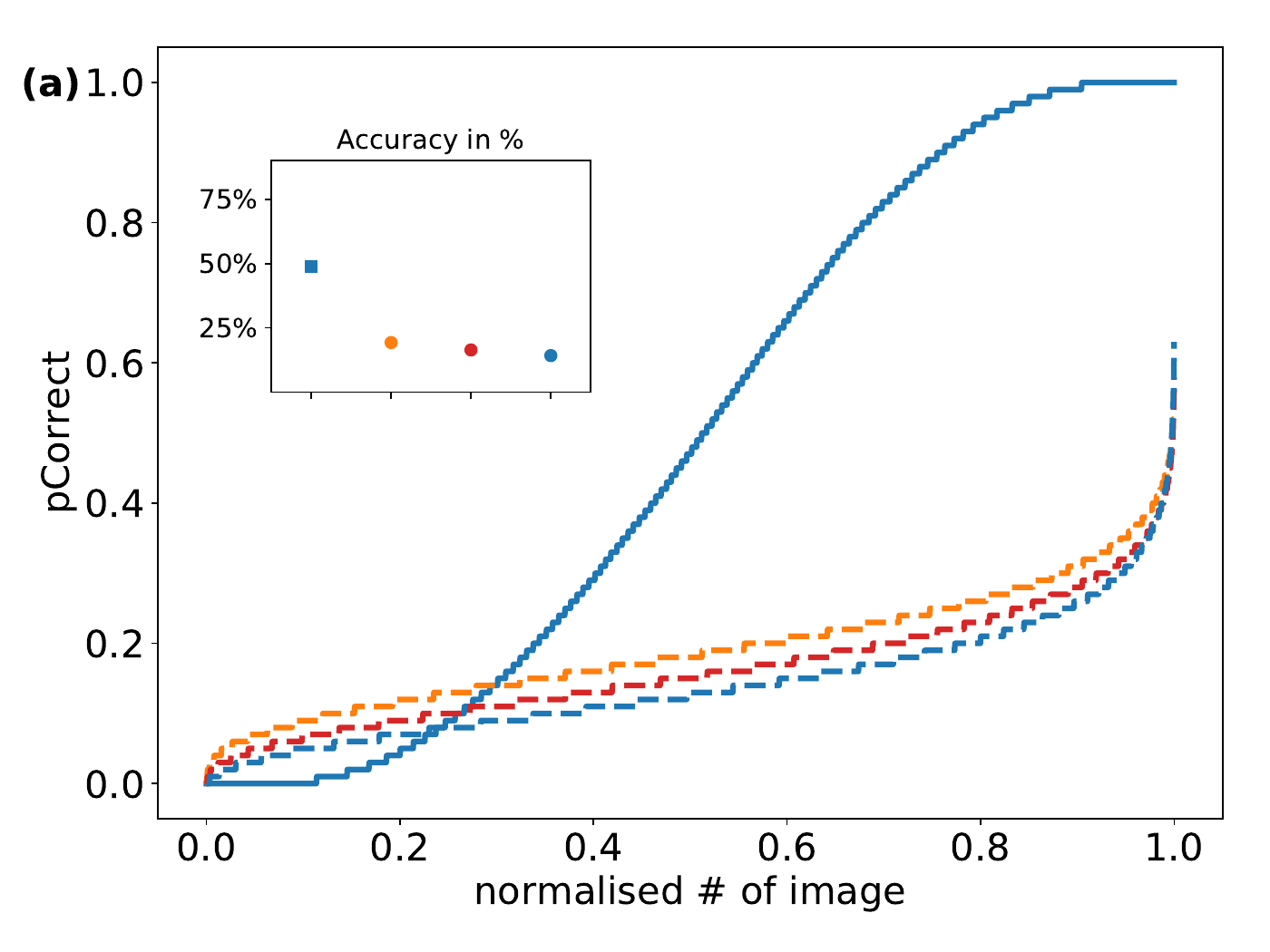}%
  \includegraphics[width=0.5\textwidth]{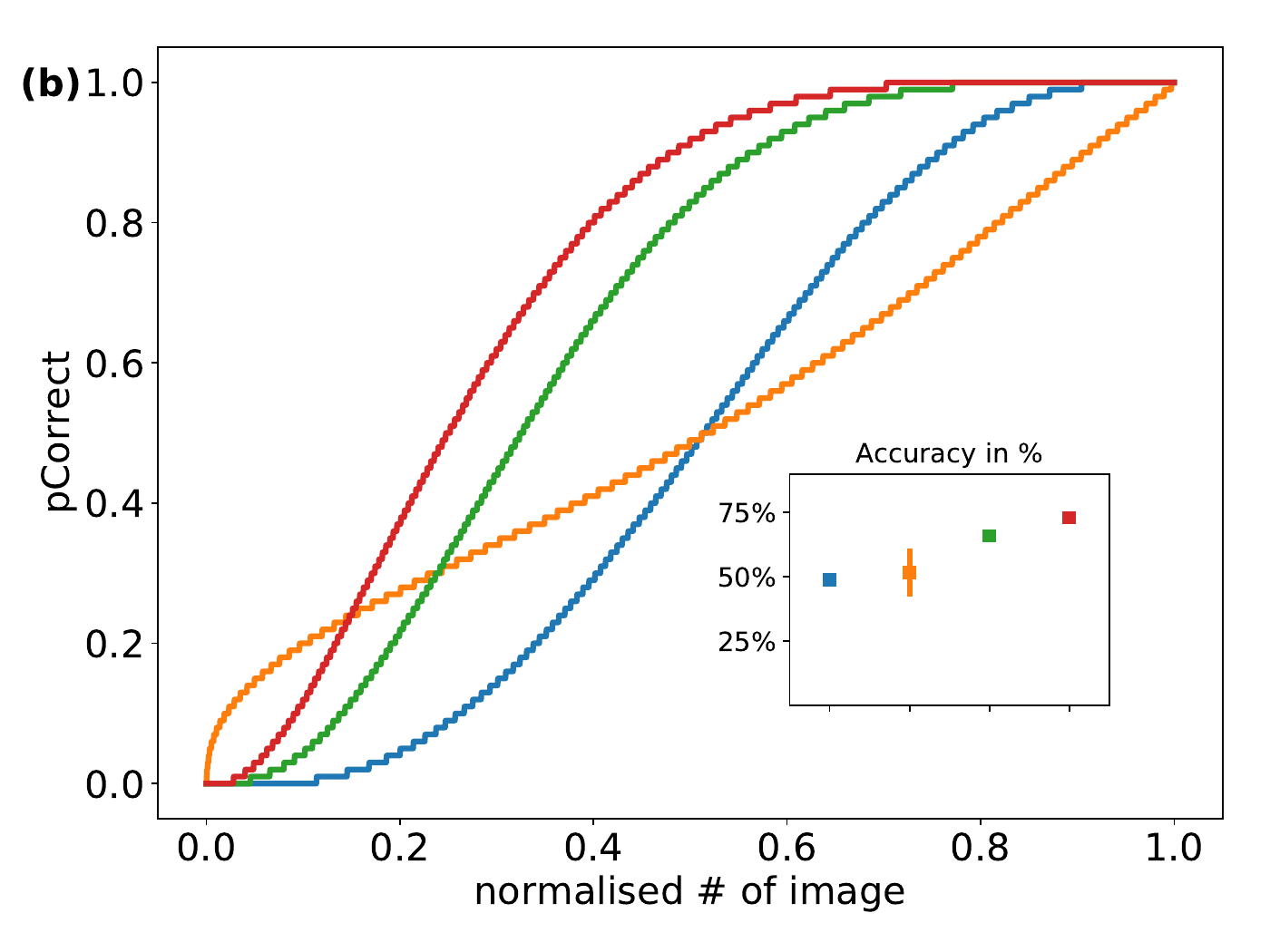}%
  \caption{\label{fig:memorisation} \textbf{Neural networks have different
      memorisation patterns for random and structured data sets.}  We plot the
    memorability of training images, i.e.\ the frequency with which an image
    from the training set is correctly classified by a neural network after
    training for only a single epoch. In Figs.~(a) and~(b), we reproduce the
    result of~\citet{arpit2017closer} for \textcolor{C0}{CIFAR10} (full blue
    line). This curve demonstrates the existence of hard and easy examples which
    are never, or always, classified correctly. Fig.~(a) shows that this
    property disappears in all models when the labels are reshuffled (dashed
    lines). The insets indicate the training accuracy after training, using
    circles for randomised data sets and squares for unmodified data
    sets. Fig.~(b) shows that these hard and easy examples also exist in the
    structured data models, \textcolor{C2}{TeacherS} and the
    \textcolor{C3}{HMM}, but not in the unstructured \textcolor{C1}{Gaussian}
    one.}
\end{figure*}

An interesting difference between random and realistic data was demonstrated in
a recent paper by \citet{arpit2017closer}. They trained 100 two-layer networks
($K=4096$ hidden units with ReLU activation, 10 output units with softmax
activation) for a single epoch on the ten-class image classification task on the
CIFAR10 data set, starting from different initial conditions each time. At the
end of training, they measured the frequency with which each individual image
was classified correctly by the network across runs, which we will call the
memorability of an image, which should be thought of as a function of the image
and the data set that contains it. We repeated this experiment on CIFAR10, and
added three different synthetic data sets: (colour codes refer to
Fig.~\ref{fig:memorisation}):
\begin{description}
\item[\textcolor{C0}{CIFAR10}] $\vx_\mu$: CIFAR10 images; $y^*_\mu\in [0, 9]$:
  CIFAR10 label giving the class of that image.
\item[\textcolor{C1}{Gaussian}] teacher acting on Gaussian inputs: $\vx_\mu$
  i.i.d.\ standard Gaussians; $ y^*_\mu = \argmax \vphi(\vx_\mu, \vtheta^*)$
\item[\textcolor{C2}{TeacherS}] teacher acting on
  \textcolor{C2}{\textbf{s}}tructured inputs $\vx_\mu=f(\mF \vc_\mu)$;
  $ y^*_\mu = \argmax \vphi(\vx_\mu, \vtheta^*)$
\item[\textcolor{C3}{HMM}] $\vx_\mu=f(\mF \vc_\mu)$;
  $ y^*_\mu = \argmax \vphi(\vc_\mu, \vtheta^*)$
\end{description}
The labels for the synthetic data sets were generated by two teacher networks,
one with input dimension $N$ for the Gaussian and TeacherS data sets, and
another with input dimension $D$ for the HMM. The teachers were two-layer fully
connected networks having $M=2K$ hidden units with ReLU activation function, and
10 nodes in the last readout layer. Thus the teacher's output
$\vphi(\cdot, \vtheta^*)\in\mathbb{R}^{10}$, and the class for a given input was
obtained as the index of the output node with the highest value for that input.

We plot the memorabilities for all images in the training set, sorted by their
memorability, in Fig.~\ref{fig:memorisation}. On the left, we first reproduce
the memorability curve for CIFAR10 that was found by~\citet{arpit2017closer}
(solid blue), which demonstrates that many examples are \emph{consistently}
classified correctly or incorrectly after a single epoch of training. The
memorability curve for a data set containing the same images with \emph{random}
labels (dashed blue) demonstrates that randomised CIFAR10 doesn't contain images
that are particularly hard or easy to memorise. The smaller variation in
memorability for the randomised data set is largely due to the fact that it
takes it more time to fit randomised data sets~\cite{Zhang2016a}. After one
epoch, the network thus has a lower training accuracy on the randomised data set
(cf.\ the inset of Fig.~\ref{fig:memorisation}), which leads to the smaller area
underneath the curve. We verified that no easy or hard samples appear when
training the randomised data sets to comparable training accuracy (not
shown). 
In fact, the memorability of data sets with random labels seem to coincide after
accounting for differences in the training error, regardless of whether the
inputs are CIFAR10 images, Gaussian inputs or structured inputs
$\mX = f(\mC \mF)$~\eqref{eq:structured-inputs} (dashed lines in
Fig.~\ref{fig:memorisation} a).

The memorability curves for the Gaussian, TeacherS and HMM data sets in
Fig.~\ref{fig:memorisation} (b) reveal that hard and easy examples exist for
TeacherS and HMM, which both contain structured inputs $\mX=f(\mC \mF)$, but not
in the Gaussian data set. The number of easy examples, but not their existence,
correlates well with the training accuracy on these data sets, shown in the
inset. In that sense, the hidden manifold model is thus a more realistic model
of image-like data than the canonical teacher-student setup.

Note that by making the teacher network larger than the students ($M=2K$), the
learning problem is unrealisable for all three synthetic data sets, \emph{i.e.}
there is no set of weights for the student that achieve zero generalisation
error. The absence of easy examples in the Gaussian data set thus suggests that
unrealisability alone is insufficient to obtain a data set with easy
examples. Our results also demonstrate that memorability is not just a function
of the input correlations: CIFAR10 images, Gaussian inputs and structured inputs
yield the same memorability curves when their labels are randomised. We leave it
to future work to identify some criterion, statistical or otherwise, that
predicts either whether a sample ($\vx_\mu, y^*_\mu$) is easy (or hard!) to
memorise or whether a training set contains easy examples at all.

\section{Concluding perspectives}

We have introduced the hidden manifold model as a generative model for
structured data sets that displays some of the phenomena that we observe when
training two-layer neural networks on realistic data. The HMM has two key
ingredients, namely high-dimensional inputs which lie on a lower-dimensional
manifold, and labels for these inputs that depend on the inputs' position within
the low dimensional manifold.  We derived an analytical solution of the model
for online SGD learning of two-layer neural networks. We thus provide a rich
test bed for exploring the influence of data structure on learning in neural
networks.

Let us close this paper by outlining several important directions in which our
work is being (or should be) extended.

\paragraph*{Comparison to more deep learning phenomenology} In the spirit of our
experiments in Section~\ref{sec:comparison}, it is of great interest to identify
more properties of learning that are consistently reproduced across experiments
with realistic data sets and network architectures, and to test whether the HMM
reproduces these observations as well. Of particular interest will be those
cases where learning on realistic data deviates from the HMM, and how we can
extend the HMM to capture these behaviours.

\paragraph*{Beyond online SGD} Our analytical results on online SGD rely on the
assumption that each new sample seen during training is conditionally
independent from the weights of the network up to that point. In practice,
samples are seen several or even many times during training, giving rise to
additional correlations. Taking those correlations into account to analyse those
cases is an important future direction. First steps towards a solution to this
challenging problem were made using the dynamical replica
method~\cite{coolen2000online,coolen2000dynamics} for two-layer networks, and
for single-layer neural networks trained using full-batch gradient descent,
where all the samples in the training set are used at every step of the
algorithm~\cite{mei2019generalization, montanari2019generalization,
  gerace2020generalisation}. Generalising these results to two-layer networks is
clearly a direction for future work as well.

\paragraph*{Learning with a multi-layer network} The present work should be
extended to learning with multi-layer networks in order to identify how depth
helps to deal with structured data. This is a serious challenge, and it remains
an open problem to find explicitly solvable models of multi-layer (non-linear)
networks even in the canonical canonical teacher-student model where inputs are
uncorrelated.

\paragraph*{Multi-layer generative model} The hidden manifold model is akin to a
single layer generator of a GAN. A natural extension would be to take a
generator with an arbitrary number of layers. Multi-layer generators are
explored in~\cite{louart2018random,seddik2020random}, whose results are
analogous to the Gaussian equivalence property and suggest that the full
solution of the online SGD or of the full-batch gradient descent might also be
within reach.

\paragraph*{Conditioning the inputs on the labels} In the HMM, the true label
$y^*$ of an input $\vx$ is conditioned on its latent representation $\vc$, i.e.\
its coordinates in the manifold. It may be more realistic to consider models
where instead, the latent representation is conditioned on the label of the
input, i.e. $p(c|y)$. A simple case of such a model that reduces to a Gaussian
mixture of two clusters was explored recently~\cite{mignacco2020role}. This is
also the point of view taken implicitly in~\cite{cohen2020separability}. More
generally, exploring different approaches to modelling realistic inputs will
allow us to better understand how data structure influences learning.

\begin{acknowledgments}
  We would like to thank Bruno Loureiro and Federica Gerace for useful
  discussions. We thank Stanisław Jastrzębski for discussing the experiments
  of~\cite{arpit2017closer}. We are grateful to the Kavli Institute For
  Theoretical Physics for its hospitality during an extended stay, during which
  parts of this work were conceived and carried out. We acknowledge funding from
  the ERC under the European Union’s Horizon 2020 Research and Innovation
  Programme Grant Agreement 714608-SMiLe, from ``Chaire de recherche sur les
  modèles et sciences des données'', Fondation CFM pour la Recherche-ENS, and
  from the French National Research Agency (ANR) grant PAIL. This research was
  supported in part by the National Science Foundation under Grant No.\ NSF
  PHY-1748958.
\end{acknowledgments}

\appendix

\begin{widetext}
\section{The Gaussian Equivalence Property}
\label{app:GEP}

\subsection{Nonlinear functions of weakly correlated Gaussian random variables}
\label{app:lemma1}

In order to derive the GEP we first establish some auxiliary lemmas
concerning the correlations between nonlinear functions of weakly correlated
random variables.
\subsubsection{Correlations of two functions}
\begin{lemma}
\label{lemma:1}
Given $n+p$ random variables organised in two vectors,
\begin{equation}
x= \begin{pmatrix} x^1 \\ . \\ . \\ x^n \end{pmatrix}, \quad
y= \begin{pmatrix} y^1 \\ . \\ . \\ y^p  \end{pmatrix},
\end{equation}
with a joint Gaussian distribution, denote by $\EE$ the
expectation with respect to this distribution. The first moments are
supposed to vanish,
\begin{equation}
  \EE x_i=0,  \quad
  \EE y_j=0,
\end{equation}
and we denote by $Q,R,\varepsilon S$ the covariances:
\begin{equation}
\EE [x_i x_j]=Q_{ij}, \quad
\EE [y_i y_j]=R_{ij}, \quad
\EE[x_i y_j]=\varepsilon S_{ij}.
\end{equation}
Let $f(x)$ and $g(y)$ be two functions of $x$ and $y$ respectively regular
enough so that $\EE_x [x_i f(x)]$, $\EE_x [x_i x_j f(x)]$, $\EE_y [y_i f(y)]$
and $\EE_y [y_i y_j f(y)]$ exist, where $ {\mathbb E}_x$ denotes the expectation
with respect to the distribution ${\mathcal N}(a,Q)$ of $x$ and $ {\mathbb E}_y$
denotes the expectation with respect to the distribution ${\mathcal N}(b,R)$ of
$x$.

Then, in the $\varepsilon \to 0$ limit:
\begin{equation}
\EE [f(x)g(y)]= \EE_x[f(x)]\;  \EE_y[g(y)]
 +\varepsilon \sum_{i=1}^n\sum_{j=1}^p {\mathbb
  E}_x[x_i f(x)] \left(Q^{-1}SR^{-1}\right)_{ij}{\mathbb
  E}_y[y_j g(y)]+\order{\varepsilon^2}\ .
  \label{lemmaresult}
\end{equation}
\end{lemma}

\begin{proof}
  The result is obtained by a straightforward expansion in
  $\varepsilon$.

  The joint distribution of $x$ and $y$ is
  \begin{equation}
    \label{Pxy}
    P(x,y)=\frac{1}{Z}\exp\left[-\frac{1}{2}
      \begin{pmatrix} x & y \end{pmatrix}
      M^{-1}
      \begin{pmatrix} x \\ y \end{pmatrix}
      \right]
    \end{equation}
    where
    \begin{equation}
      M=\begin{pmatrix} Q & \varepsilon S \\ \varepsilon S^T &
        R\end{pmatrix} \, .
    \end{equation}
   One can expand the inverse matrix $M^{-1}$ to first order in
    $\varepsilon$:
    \begin{equation}
      M^{-1}=
       \begin{pmatrix} Q^{-1} & 0 \\ 0 &
         R^{-1}\end{pmatrix}
       -\varepsilon
          \begin{pmatrix} 0 & Q^{-1}S R^{-1} \\ R^{-1} S^T Q^{-1}  & 0 \end{pmatrix}
     \end{equation}
     and substitute this into the joint distribution~\eqref{Pxy} to find
          \begin{equation}
    P(x,y)= \frac{1}{Z}\exp\left[-\frac{1}{2}
      \begin{pmatrix} x&y \end{pmatrix}
     \begin{pmatrix} Q^{-1} & 0 \\ 0 &
         R^{-1}\end{pmatrix}
      \begin{pmatrix} x\\ y\end{pmatrix}
            \right]
\left[1+\varepsilon \sum_{i=1}^n\sum_{j=1}^p
  x_i\left(Q^{-1}SR^{-1}\right)_{ij}y_j+\order{\varepsilon^2}\right].
\end{equation}
Using this expression, the result~\eqref{lemmaresult} follows immediately.
\end{proof}

An immediate application of the lemma to the case when $n=p=1$ is the
following. Consider two Gaussian random variables $u_1,u_2$ with mean
zero and covariance
\begin{equation}
  \EE [u_1^2]=1 \ \ \ ; \ \ \  \EE [u_2^2]=1 \ \ \ ; \ \ \
  \EE[ u_1 u_2]=\varepsilon m_{12}\ ,
  \label{u1u2def}
\end{equation}
and  two functions $f_1$ and $f_2$. Define, for $i \in \{1,2\}$:
\begin{equation}
  a_i= \langle f_i(u)\rangle \ \ \ ; \ \ \  b_i= \langle u
  f_i(u)\rangle 
  \label{abdef}
  \end{equation}
  where  $\langle .\rangle$ denotes the average over the distribution of
  the random Gaussian variable $u$ distributed as ${\cal N}
  (0,1)$.
  
Then, in the
$\varepsilon\to 0$ limit, the correlation between $f(u_1)$ and
$g(u_2)$ is given by
\begin{equation}
  \EE [f_1(u_1) f_2(u_2)] = a_1 a_2
  +\varepsilon   m_{12} b_1 b_2+\order{\varepsilon^2} \, .
\end{equation}
This means that, if we consider centered functions
  $\tilde f_i(u_i)= f_i(u_i)-a_i$,
their covariance is
\begin{equation}
  \label{f1f2corr}
  \EE [\tilde f_1(u_1) \tilde f_2(u_2)] =
  +\varepsilon   m_{12} b_1 b_2+\order{\varepsilon^2} \, .
\end{equation}
This result generalises to correlation functions of higher order, as
stated in the following lemma.

\subsubsection {Higher-order correlations}
\begin{lemma}
  \label{lemma:2}
Consider $m$ Gaussian random variables $u_1,\dots,u_m$ with mean
zero and covariance
\begin{equation}
 \forall i:\  \EE [u_i^2]=1, \quad \forall i\neq j:\ 
  \EE[ u_i u_j]=\varepsilon m_{ij},
\end{equation}
and $m$ functions $f_1, \ldots, f_m$.
Define as before:
\begin{equation}
  a_i= \langle f_i(u)\rangle,\quad b_i= \langle u
  f_i(u)\rangle ,\ \ i\in\{1,\dots,m\}
\end{equation}
and define the centered functions as
\begin{equation}
  \tilde f_i(u)=f_i(u)-a_i \ ,
\end{equation}
then
\begin{equation}
  \label{lemma2eq}
  \lim_{\varepsilon \to 0} \frac{1}{\varepsilon^{m/2}} \EE
  \tilde  f_1(u_1)\dots \tilde f_m(u_m)
  = \begin{cases}
    \sum_{\sigma \in \Pi}m_{\sigma_1\sigma_2}
    m_{\sigma_{p-1}}m_{\sigma_p} & \text{if p is even}\\
    = 0& {\text{if p is odd}}
  \end{cases}
\end{equation}

where $\Pi$ denotes all the $m!/(2^{m/2}(m/2)!)$ partitions of
$\{1,\dots,m\}$ into $m/2$ disjoint pairs. This result means that, for
the moments involving only different indices, the random variables
$\tilde f_1(u_1)/\sqrt{\varepsilon},\dots, \tilde f_m(u_m)/\sqrt{\varepsilon}$
behave, in the $\varepsilon\to 0$ limit, like Gaussian variables with
a covariance matrix $b_i b_j m_{ij}$.
\end{lemma}

\begin{proof}
The covariance matrix $U$ of the variables $u_1,\dots,u_m$ has
elements $1$ on the diagonal, and elements of order $\varepsilon$ out
of the diagonal: $U={\mathbb{I}}+\varepsilon m$. One can expand
$U^{-1}$ in powers of $\varepsilon$:
\begin{equation}
U^{-1}=\sum_{p=0}^\infty (-\varepsilon)^p m^p\, .
\end{equation}
The integration measure over the variables $u_1,\dots, u_m$ can be
expanded as:
\begin{equation}
  \sqrt{(2\pi)^m {\text{det}}M}\; e^{-\frac{1}{2}\sum_i u_i^2}
  \prod_{p=1}^\infty G_p(u_1,\dots,u_m)
  \label{measure_expanded}
\end{equation}
where
\begin{equation}
  G_p(u_1,\dots,u_m)=
  1 +\left(-\frac{\varepsilon}{2}\right)^p \sum_{ij}(m^p)_{ij}u_i u_j+
  \frac{1}{2!}\left(-\frac{\varepsilon}{2}\right)^{2p}
  \sum_{ijk\ell}(m^p)_{ij}(m^p)_{k\ell}u_i u_j u_k u_\ell+\dots
\end{equation}
When we compute the integral of $\tilde f_1(u_1)\dots \tilde f_m(u_m)$ with the
measure~\eqref{measure_expanded}, because of the fact that
$\langle \tilde f_i(u_i)\rangle=0$, we need to include terms coming from
$\prod_p G_p(u_1,\dots,u_p)$ that involve at least one power of each of the
variables $u_1,\dots,u_m$.

When $m$ is even, say $m=2r$, for $\varepsilon\to 0$, the term of this kind with
the smallest power of $\varepsilon$ is the monomial $u_1\dots u_{2r}$ that comes
from the $r$th order term in $G_1$. This gives:
\begin{equation}
\EE f_1(u_1)\dots f_{2r}(u_{2r})=
\frac{1}{r!}\left(\frac{\varepsilon}{2}\right)^r\hat \sum_{i_1j_1\dots
  i_r j_r} m_{i_1j_1}m_{i_r j_r}+\order{\varepsilon^{r+1}}\ ,
\end{equation}
where the sum $\hat \sum_{i_1j_1\dots i_r j_r}$ runs over all
permutations of the indices $1,\dots,2r$. This leads to (\ref{lemma2eq})
for $m$ even.

When $m$ is odd, $m=2r+1$, for $\varepsilon\to 0$, the leading terms coming from
$\prod G_p$ that give a non-zero result are monomials of the type
$u_1^1 u_2\dots u_{2r+1}$. They are of order $\order{\varepsilon^{r+1}}$. This
proves (\ref{lemma2eq}) for $m$ odd.
\end{proof}

\begin{corollary}
  In the special case $m=3$, we get
  \begin{equation}
    \label{3fcorr}
    \EE [f_1(u_1)f_2(u_2)f_3(u_3)]= a_1a_2a_3+\varepsilon (a_1
    m_{23}b_2b_3+a_2 m_{13}b_1b_3+a_3m_{12}b_1b_2) \, .
  \end{equation}
\end{corollary}

\subsection{Derivation of the Gaussian Equivalence Property}

The derivation is based on the computation of moments of the variables $\lambda^k$
and $\nu^m$, showing that, in the thermodynamic limit, all the moments are those
of Gaussian random variables. Here we shall explicit the derivation up to fourth
order moments, and leave the daunting higher order moments for a
future formal proof of the GEP.
  
\subsubsection{Covariances}
  
We first compute the covariance matrix
$G^{k \ell}= {\mathbb E}[\tilde \lambda^k\tilde \lambda^\ell]$:

\begin{eqnarray}
G^{k \ell}&=&\usN \sum_{i,j} w_i^k w_j^\ell
 \EE ( f(u_i)-a)( f(u_j)-a) \\
          &=&(c-a^2)W^{k\ell}+
              \usN \sum_{i\neq j} w_i^k w_j^\ell
 \EE ( f(u_i)-a)( f(u_j)-a).
 \label{G1}
\end{eqnarray}
In the last piece, we need to compute $ \EE[ ( f(u_i)-a)( f(u_j)-a)]$ for two Gaussian random variables $u_i$ and $u_j$ which
are weakly correlated in the large $N$ limit. In fact, as $i\neq j$:
\begin{equation}
 \EE u_i u_j = U_{ij}
 \end{equation}
is of order $1/\sqrt{D}$. In the thermodynamic limit, we can apply the
lemma (\ref{lemma:1}) which gives:
\begin{equation}
 \EE f(u_i) f(u_j) = a^2 +b^2 \usR \sum_{r=1}^D F_{ir}F_{jr} \qquad (i\neq j).
 \label{G2}
\end{equation}
From Eqns.~\eqref{G1} and~\eqref{G2}, we get the covariance of $\lambda$
variables as written in (\ref{eq:Q}).  The covariance $ \EE [\nu^m \nu^n]$ is
analogous.

We now compute the covariance $ \EE [\tilde \lambda^k\nu^m]$.
This is equal to
\begin{equation}
 \ussqN \sum_{i=1}^N w_i^k \ussqR \sum_{r=1}^D \tilde w^m_r \; \EE [
 f(u_i) C_r]\, .
 \label{intermed1}
\end{equation}
The two variables $u_i$ and $c_r$ are Gaussian random variables with a
correlation 
 \begin{equation}
   \EE [u_i C_r]= \ussD F_{ir}
 \end{equation}
 which goes to zero as $\order{1/\sqrt{N}}$ in the thermodynamic limit. We
 can thus use Lemma (\ref{lemma:1}), and more precisely
 Eq. (\ref{f1f2corr}), to get

\begin{equation}
  \EE [
  f(u_i) C_r]= \ussR F_{ir}\langle u f(u)\rangle \langle C_r^2\rangle=
  \frac{b}{\sqrt{D}} F_{ir} \, .
 \end{equation}
 Using this result in (\ref{intermed1}) gives Eq.~\eqref{eq:R}.
 
 \subsubsection{ Fourth moments of $\tilde \lambda^k$ variables}

 We study the fourth moment defined as:
 \begin{equation}
   \label{lambda-four}
   G^{k_1k_2k_3k_4}=\langle \tilde \lambda^{k_1} \tilde \lambda^{k_2} \tilde \lambda^{k_3}
   \tilde \lambda^{k_4}\rangle=
   \frac{1}{N^2}\sum_{i_1,i_2,i_3,i_4} w_{i_1}^{k_1} w_{i_2}^{k_2} w_{i_3}^{k_3} w_{i_4}^{k_4}
   \langle \tilde  f(u_{i_1}) \tilde  f(u_{i_2}) \tilde  f(u_{i_3}) \tilde  f(u_{i_4})\rangle
\end{equation}
where $\tilde f(u)=f(u)-a$ is the centered function.

We shall decompose the sum over $i_1,i_2,i_3,i_4$ depending on the
number of distinct indices there are.

\paragraph{Distinct indices} Let us study the first piece of the fourth moment
$\langle\lambda^{k_1}\lambda^{k_2}\lambda^{k_3}\lambda^{k_4}\rangle$:
\begin{equation}
G_4^{k_1k_2k_3k_4}= 
\frac{1}{N^2}\sum'_{i_1,i_2,i_3,i_4} w_{i_1}^{k_1} w_{i_2}^{k_2} w_{i_3}^{k_3} w_{i_4}^{k_4}
\langle \tilde  f(u_{i_1}) \tilde  f(u_{i_2}) \tilde  f(u_{i_3}) \tilde  f(u_{i_4})\rangle
\end{equation}
where the sum runs over four indices $i_1,i_2,i_3,i_4$ which are distinct from
each other. We can use the factorisation property of the 4th moments of $f(u)$
of lemma (\ref{lemma:2}).  This gives
\begin{eqnarray}\nonumber
  G_4^{k_1k_2k_3k_4}&=&\frac{1}{N^2}\sum'_{i_1,i_2,i_3,i_4} w_{i_1}^{k_1} w_{i_2}^{k_2} w_{i_3}^{k_3} w_{i_4}^{k_4}
                        \left[\langle \tilde  f(u_{i_1}) \tilde  f(u_{i_2}) \rangle \langle \tilde  f(u_{i_3})
                        \tilde  f(u_{i_4})\rangle+{\text{2 perm.}}\right]\\
                    &=&\left(  \left[
                        \usN\sum'_{i_1,i_2} w_{i_1}^{k_1}w_{i_2}^{k_2}\langle f(u_{i_1}) f(u_{i_2})\rangle
                        \right]
                        \left[\usN\sum'_{i_3,i_4} w_{i_3}^{k_3}
                        w_{i_4}^{k_4}\langle f(u_{i_3}) f(u_{i_4}) \rangle\right]
                        -{\text{Corr.}}\right)\nonumber\\
                    &&+{\text{2 perm.}} \, .
\end{eqnarray}
The correction terms come from pieces where the intersection between
$\{i_1,i_2\}$ and $\{i_3,i_4\}$ is non-empty. If we first neglect this
correction, we find
\begin{eqnarray}
   \label{lambda-four-res-final}
G_4^{k_1k_2k_3k_4}= b^4\left[
  \left(\Sigma^{k_1k_2}-W^{k_1k_2}\right) 
  \left(\Sigma^{k_3k_4}-W^{k_3k_4}\right) +{\text{2 perm.}}\right] \, .
\end{eqnarray}

Now we shall show that the corrections are negligible. Consider the term
$i_1=i_3$, $i_2\neq i_4$. This gives a correction
\begin{equation}
-\frac{1}{N^2}\sum'_{i_1,i_2,i_4} w_{i_1}^{k_1} w_{i_2}^{k_2}
w_{i_1}^{k_3} w_{i_4}^{k_4}
\left[\langle \tilde  f(u_{i_1}) \tilde  f(u_{i_2}) \rangle \langle \tilde  f(u_{i_1})
                   \tilde  f(u_{i_4})\rangle\right] \, .
\end{equation}
Using~\eqref{f1f2corr}
\begin{equation}
\langle \tilde f(u_{i_1}) \tilde f(u_{i_2}) \rangle=b^2U_{i_1i_2}=
b^2\frac{1}{D} \sum_{r=1}^D F_{i_1 r}F_{i_2 r}\ ,
  \end{equation}
  we get the expression for the correction
  \begin{equation}
    -\frac{1}{N^2D^2 } \langle
    uf(u)\rangle^4\sum'_{i_1,i_2,i_4} w_{i_1}^{k_1} w_{i_2}^{k_2}
    w_{i_1}^{k_3} w_{i_4}^{k_4}F_{i_1 r}F_{i_2 r}F_{i_1 s}F_{i_4 s}=
    -\frac{1}{\sqrt{N} D^2 }\sum_{r,s}S^{k_1k_3}_{rs} S^{k_2}_r S^{k_4}_s.
\end{equation}
Using our hypothesis on the fact that the quantities $S$ are of order one, this
correction is clearly at most of order $\order{1/\sqrt{N}}$, and therefore
negligible.

The last correction that we need to consider is the term where
$i_1=i_3=i$, and $i_2=i_4=j$.
This gives
\begin{equation}
-\frac{1}{N^2}\sum'_{i,j} w_{i}^{k_1} w_{j}^{k_2}
w_{i}^{k_3} w_{j}^{k_4}
\langle \tilde f(u_{i}) \tilde f(u_{j}) \rangle^2
=
  -\frac{1}{NR^2} \langle
  uf(u)\rangle^4\sum_{r,s}\left[S^{k_1k_3}_{rs}S^{k_2k_4}_{rs}-S^{k_1k_3k_2k_4}_{rrss}\right]\ ,
\end{equation}
which is again negligible in the large $N$ limit.

\paragraph{Three distinct indices} Let us study the contributions to the fourth
moment of $\lambda$ coming from three distinct indices. We study the case where
$i_1=i_4$:
\begin{equation}
  \label{lambda-three}
E^{k_1k_2k_3k_4}= 
\frac{1}{N^2}\sum'_{i_1,i_2,i_3} w_{i_1}^{k_1} w_{i_2}^{k_2} w_{i_3}^{k_3} w_{i_1}^{k_4}
\langle \tilde f(u_{i_1})^2 \tilde f(u_{i_2}) \tilde f(u_{i_3})\rangle
\, .
\end{equation}
Using the expression for the third moment of functions of $u_1,u_2,u_3$ found in (\ref{3fcorr}), we get:
\begin{eqnarray}\nonumber
E^{k_1k_2k_3k_4}&= &
c b^2 \frac{1}{N^2}\sum'_{i_1,i_2,i_3} w_{i_1}^{k_1} w_{i_2}^{k_2} w_{i_3}^{k_3} w_{i_1}^{k_4}
                     - {\text{Corr.}}
  \nonumber \\
  &=&
c b^2
      W^{k_1k_4}\left[\Sigma^{k_2k_3}-W^{k_2k_3}\right]
      - {\text{Corr.}}
  \label{lambda-three-res}
\end{eqnarray}
The corrections come from cases when $i_1=i_2$ or $i_1=i_3$. For instance the
piece with $i_1=i_2$ gives
\begin{equation}
  - c b^2\frac{1}{N D}\sum_r S^{k_1k_2k_4}_{r}S^{k_3}_r
\end{equation}
which is $\order{1/N}$ at most.

The only pieces that do not vanish in the large $N$ limit are thus the
pieces similar to the one computed in (\ref{lambda-three-res}). Putting
all of them together we find that the contribution to $\langle
\tilde \lambda^{k_1}\tilde \lambda^{k_2}\tilde \lambda^{k_3}\tilde \lambda^{k_4}\rangle$ coming
form pieces with exactly three distinct indices in $i_1,i_2,i_3,i_4$ is
equal to:
\begin{eqnarray}\nonumber
  G_3^{k_1k_2k_3k_4} =c b^2 \big(
  X^{k_1k_2;k_3k_4}+
  X^{k_1k_3;k_2k_4}+
  X^{k_1k_4;k_2k_3}+
     X^{k_2k_3;k_1k_4}+
  X^{k_2k_4;k_1k_3}+
  X^{k_3k_4;k_1k_2}
     \big)
     \label{lambda-three-res-final}
\end{eqnarray}
where
\begin{equation}
  X^{k_1k_2;k_3k_4} =
  W^{k_1k_2}\left[\Sigma^{k_3k_4}-W^{k_3k_4}\right] \, .
  \label{Xdef}
\end{equation}

\paragraph{Two distinct indices} Let us now study the contribution to the fourth
moment of $\lambda$ coming from two distinct indices.  We study first one piece
of this contribution to the fourth moment, corresponding to $i_1=i_2=i$,
$i_3=i_4=j$:
\begin{equation}
  \label{lambda-two-a}
F^{k_1k_2k_3k_4}= 
\frac{1}{N^2}\sum'_{i,j} w_{i}^{k_1} w_{i}^{k_2} w_{j}^{k_3} w_{j}^{k_4}
\langle \tilde f(u_{i})^2 \tilde f(u_{j})^2\rangle \, .
\end{equation}
To leading order in the thermodynamic limit, we can write
\begin{equation}
  \langle \tilde f(u_{i})^2 \tilde f(u_{j})^2\rangle= c^2
\end{equation}
and therefore
\begin{equation}
  \label{res-2-indices}
  F^{k_1k_2k_3k_4}=c^2 W^{k_1k_2}W^{k_3k_4}
\end{equation}
(the correction coming from $i=j$ being obviously at most $\order{1/N}$).

We study now the second piece of this contribution to the fourth moment,
corresponding to $i_1=i_2=i_3=i$, $i_4=j$. This is equal to
\begin{equation}
  \label{lambda-two-b}
\frac{1}{N^2}\sum'_{i,j} w_{i}^{k_1} w_{i}^{k_2} w_{i}^{k_3} w_{j}^{k_4}
\langle\tilde  f(u_{i})^3 \tilde f(u_{j})\rangle \, .
\end{equation}
Using
\begin{equation}
\langle \tilde f(u_{i})^3 \tilde f(u_{j})\rangle= b \langle u \tilde f(u)^3\rangle \frac{1}{D}\sum_r F_{ir}F_{jr}\ ,
\end{equation}
this gives
\begin{equation}
  b \langle u \tilde f(u)^3\rangle \frac{1}{ND}\sum_r S^{k_1k_2k_3}_r S^{k_4}_r
\end{equation}
and it is therefore negligible.

Therefore all the contributions to the fourth moment of $\lambda$
coming from exactly two distinct indices are of the type
(\ref{res-2-indices}). They give a total contribution:
\begin{equation}
  \label{lambda-two-res-final}
G_2^{k_1k_2k_3k_4}= c^2
\left[W^{k_1k_2}W^{k_3k_4}+W^{k_1k_3}W^{k_2k_4}+W^{k_1k_4}W^{k_2k_4}\right]
\, .
\end{equation}

\paragraph{One distinct index} The contribution to the fourth moment
$\langle\lambda^{k_1}\lambda^{k_2}\lambda^{k_3}\lambda^{k_4}\rangle$ coming from
$i_1=i_2=i_3=i_4$ is clearly of $\order{1/N}$ and can be neglected.

\paragraph{Final result for the four-point correlation function of $\lambda$
  variables} We can now put together all the contributions to the fourth moment
$\langle \tilde \lambda^{k_1}\tilde \lambda^{k_2}
\tilde\lambda^{k_3}\tilde\lambda^{k_4}\rangle$ coming form pieces with four
distinct indices found in~\eqref{lambda-four-res-final}, those with three
distinct indices found in~\eqref{lambda-three-res-final}, and those with two
distinct indices found in~\eqref{lambda-two-res-final}.  Defining
\begin{equation}
  Y^{k_1k_2}=\Sigma^{k_1k_2}-W^{k_1k_2},
\end{equation}
and recalling the definition~\eqref{Xdef} of the $X$ variables, we obtain:
\begin{align}\nonumber
  \langle \tilde \lambda^{k_1}\tilde \lambda^{k_2}\tilde \lambda^{k_3}\tilde \lambda^{k_4}\rangle
  &=
      b^4 \big( Y^{k_1k_2}Y^{k_3k_4}+
      Y^{k_1k_3}Y^{k_2k_4}+Y^{k_1k_4}Y^{k_2k_3}\big)\\
  &\quad + b^2 c
      \big(
      X^{k_1k_2;k_3k_4}+
      X^{k_1k_3;k_2k_4}\nonumber \\
  &\quad + X^{k_1k_4;k_2k_3}+
      X^{k_2k_3;k_1k_4}+
      X^{k_2k_4;k_1k_3}+
      X^{k_3k_4;k_1k_2}
      \big) \nonumber \\
  &\quad +c^2
      \left[W^{k_1k_2}W^{k_3k_4}+W^{k_1k_3}W^{k_2k_4}+W^{k_1k_4}W^{k_2k_4}\right].
\end{align}
We can see that this is equal to
 \begin{eqnarray}
   \left[b^2
   Y^{k_1k_2}+c
   W^{k_1k_2}\right]
   \left[b^2
   Y^{k_3k_4}+c
   W^{k_3k_4}\right]+{\text{2 perm.}}
 \end{eqnarray}
 which shows that
\begin{equation} 
  \langle \tilde \lambda^{k_1} \tilde \lambda^{k_2} \tilde \lambda^{k_3}
  \tilde \lambda^{k_4}\rangle=\langle \tilde \lambda^{k_1} \tilde \lambda^{k_2}\rangle \langle \tilde \lambda^{k_3}
  \tilde \lambda^{k_4}\rangle+{\text{2 permutations}} \, .
\end{equation}
With this, it is clear how to proceed with the calculation of the fourth moments
involving $\lambda$ and~$\nu$ variables. We first need to study the moments with
three $\lambda$ and one $\nu$, then moments with two $\lambda$ and two $\nu$,
and finally the moments with one $\lambda$ and three $\nu$ variables. In the
interest of conciseness, we do not spell out the full details of this
calculations here, which proceeds very similarly to the calculations performed
hitherto.


The generalisation to higher moments of $\lambda$ variables employs the same
combination of repeated use of Lemma~\ref{lemma:2} and careful decomposition in
subsets of distinct indices. As a result, it can be seen that the set of
$\lambda$ variables has a Gaussian distribution in the thermodynamic limit.

\section{Derivation of the equations of motion}
\label{app:derivation}

When we make a step of SGD, we update the weight $w_i^k$ using a new sample,
generated using a previously unused sample according to
\begin{subequations}
  \label{eq:sgd}
  \begin{align}
    \left(w_i^k\right)_{\mu+1}- \left(w_i^k\right)_{\mu}
    & =-\frac{\eta}{\sqrt{N}} v^k \Delta g'(\lambda^k) f(u_i), \label{eq:sgdw}\\
    v^k_{\mu+1} - v^k_\mu &= - \frac{\eta}{N} g(\lambda^k) \Delta, \label{eq:sgdv}
  \end{align}
\end{subequations}
where
$\Delta = \sum_{j = 1}^K v^j g(\lambda^j) -\sum_{m=1}^M \tilde v^m
\tilde{g}(\nu^m)$. Note the different re-scaling of the learning rate for the
first and second-layer weights. From here on out, we shall drop the index $\mu$
on the right-hand side as we work at a fixed iteration time. We will keep the
convention of Sec.~\ref{sec:GEP} where extensive indices (taking values up to
$N$ or $D$) are below the line, while we'll use upper indices when they take
finite values up to $M$ or $K$. The challenge of controlling the learning in the
thermodynamic limit will be to write closed equations using matrices with only
``upper'' indices left. Furthermore, we will adopt the convention that the
indices $j,k,\ell,\iota=1,\ldots,K$ always denote \emph{student} nodes, while
$n,m=1,\ldots,M$ are reserved for teacher hidden nodes.

\subsection{First steps}

We will start by focussing on the dynamics of the first layer,
Eq.~\eqref{eq:sgdw}. When we study the evolution of quantities that are linear
in the weights, like $S_r^k$ and the order parameters constructed from it,
\emph{e.g.}  $\Sigma^{k\ell}$, we need to study
\begin{equation}
  \left[ \sum_{j=1}^K v^j g(\lambda^j) -\sum_{m=1}^M\tilde v^m \tilde g(\nu^m)\right]
  g'(\lambda^k) f(u_i) = \sum_{j \neq k}^K v^j \cala^{jk}_i +
  v^k \calb^k_{i}-\sum_{n=1}^M \tilde v^n \calc^{nk}_i,
\end{equation}
where
\begin{align}
  \cala^{j k}_ i&= g(\lambda^j) g'(\lambda^k) f(u_i)\, ,\\
  \calb^k_{i}&=  g(\lambda^k) g'(\lambda^k) f(u_i)\, ,\\
  \calc^{nk}_i &= \tilde g(\nu^n) g'(\lambda^k) f(u_i) \, .
\end{align}
while we keep the second-layer weights $v^k$ fixed. We can thus follow the
dynamics of $S_r^k$~\eqref{eq:S}, which is linear in the weights and enters the
definition of the order parameters $R^{km}$~\eqref{eq:R} and
$\Sigma^{kl}$~\eqref{eq:Sigma}:
\begin{equation}
  \label{eq:eom-S-with-F}
  \left(S_r^k\right)_{\mu+1}-\left(S_r^k\right)_\mu=-\frac{\eta}{N}v^k \sum_i F_{ir} 
  \left[\sum_{j\neq k}^K v^j \cala^{jk}_i + v^k \calb^k_{i}-\sum_n^M \tilde v^n \calc^{nk}_i\right].
\end{equation}
We want to average this update equation over a new incoming sample, i.e.\ over
the $\vc_r$ variables. Upon contraction with $F_{ir}$ in
Eq.~\eqref{eq:eom-S-with-F}, we are thus led to computing the averages
\begin{gather}
  \mathcal{A}^{jk}_r \equiv \ussqN \sum_i \EE\left[ F_{ir} \cala_i^{jk} \right]= \EE\left[ g(\lambda^j) g'(\lambda^k) \beta_r \right],\label{eq:Ajkr}\\
  \mathcal{B}^k_r \equiv \EE\left[ g(\lambda^k) g'(\lambda^k) \beta_r
  \right], \label{eq:Bkr}
\end{gather}
and
\begin{equation}
  \label{eq:Cnkr}
  \mathcal{C}^{nk}_r = \EE\left[\tilde g(\nu^n) g'(\lambda^k)
    \beta_r \right] \ ,
\end{equation}
where
\begin{equation}
  \label{def:betar}
  \beta_r= \frac{1}{\sqrt N} \sum_i F_{ir}f(u_i).
\end{equation}
The crucial fact that allows for an analytic study of online learning is that,
at each step $\mu$ of SGD, a previously unseen input $\vx_\mu$ is used to
evaluate the gradient. The latent representation $\vc_\mu$ of this input is
given by a new set of i.i.d.\ Gaussian random variables $c_{\mu r}$, which are
thus independent of the current weights of the student at that time. In the
thermodynamic limit, the GEP of the previous section shows that, for one given
value of $r$, the $K + M + 1$ variables $\{\lambda^k\}$, $\{\nu^m\}$ and
$\beta_r$ have a joint Gaussian distribution, making it possible to express the
averages over $\{\lambda^k, \nu^m, \beta_r\}$ in terms of only their
covariances.

Looking closer, we see that the average of
(\ref{eq:Ajkr},\ref{eq:Bkr},\ref{eq:Cnkr}) over this Gaussian distribution
involves two sets of random variables: on the one hand, the $M+K$ local fields
$\{\nu^m, \lambda^k\}$, which have correlations of order $1$, and on the other
hand the variable $\beta_r$ (for one given value of $r$). It turns out that
$\beta_r$ is only weakly correlated with the local fields $\{\nu^m, \lambda^k\}$
(the correlation is $\mathcal{O}(1/\sqrt N)$). In Appendix~\ref{app:lemma1}, we
discuss how to compute this type of average and prove Lemma~\ref{lemma:1}, which
for the averages~(\ref{eq:Ajkr}--\ref{eq:Cnkr}) yields
\begin{align}
  \begin{split}
    \mathcal{A}^{jk}_r &=\frac{1}{Q^{kk}Q^{j
        j}-(Q^{kj})^2}  \left( Q^{jj} \EE\left[ g'(\lambda^k)\lambda^k
    g(\lambda^j) \right] \;
    \EE\left[ \lambda^k \beta_r\right] -Q^{kj} \EE\left[ g'(\lambda^k)\lambda^j g(\lambda^j) \right] \;
    \EE\left[ \lambda^k \beta_r\right] \right.  \\
    & \hspace*{7em}\left. -Q^{kj} \EE\left[ g'(\lambda^k) \lambda^k g(\lambda^j) \right] \;
    \EE\left[ \lambda^j \beta_r\right]  +Q^{kk} \EE\left[ g'(\lambda^k) \lambda^j g(\lambda^j) \right] \;
    \EE\left[ \lambda^j \beta_r\right] \right),
  \end{split}\\
  \mathcal{B}^k_{r}&  = \frac{1}{Q^{kk}}
                         \EE\left[ g'(\lambda^k)\lambda^k g(\lambda^k) \right] \; 
                         \EE\left[ \lambda^k \beta_r \right],\\
  \begin{split}
    \mathcal{C}^{nk}_r &= \frac{1}{Q^{kk}T^{nn}-{\left(R^{kn}\right)}^2} 
    \left( T^{nn} \EE\left[ g'(\lambda^k)\lambda^k \tilde g(\nu^n) \right] \;
    \EE\left[ \lambda^k \beta_r\right]   -R^{kn} \EE\left[ g'(\lambda^k) \nu^n \tilde g(\nu^n) \right] \;
    \EE\left[ \lambda^k \beta_r\right]  \right.\\
    & \hspace*{7em} \left. -R^{kn} \EE\left[ g'(\lambda^k)\lambda^k \tilde g(\nu^n) \right] \;
    \EE\left[ \nu^n \beta_r\right] + Q^{kk} \EE\left[ g'(\lambda^k) \nu^n \tilde g(\nu^n) \right] \; \EE\left[
    \nu^n \beta_r\right] \right).
  \end{split}
\end{align}
This yields
\begin{equation}
  \label{eq:eom-S}
  \left(S_r^k\right)_{\mu+1}-\left(S_r^k\right)_\mu=-\frac{\eta}{\sqrt{N}}v^k
  \left[\sum_{j\neq k}^K v^j \mathcal{A}^{jk}_r + v^k \mathcal{B}^k_r - \sum_n^M
    \tilde v^n \mathcal{C}^{nk}_r\right],
\end{equation}
with only the single extensive index $r$ left. While this equation would appear to
open up a way to write down the equation of motion for the ``teacher-student''
overlap $R^{km}$ by contracting~\eqref{eq:eom-S} with $\tilde w_r^m$, we show in
Appendix~\ref{sec:cannot-close} that such a program will lead to an infinite
hierarchy of equations. To avoid this problem, we rotate the problem to a
different basis, as we explain in the next section.

\subsection{Changing the basis to close the equations}
\label{sec:rotating}

We can close the equations for the order parameters by studying their dynamics
in the basis given by the eigenvectors of the operator
\begin{equation}
  \label{eq:Omega}
  \Omega_{rs}\equiv\usN \sum_i F_{ir} F_{is},
\end{equation}
which is a $D\times D$ symmetric matrix, with diagonal elements $\Omega_{rr}=1$,
and off-diagonal elements of order $1/\sqrt{N}$. Consider the orthogonal basis
of eigenvectors $\psi_{\tau=1,\dots,D}$ of this matrix, with corresponding
eigenvalues $\rho_\tau$, such that
\begin{equation}
\sum_s \Omega_{rs}\psi_{\tau s}=\rho_\tau \psi_{\tau r}.
\end{equation}
We will suppose that the components of the eigenvectors $\psi_{\tau r}$ are of
order 1 and we impose the following normalisation:
\begin{equation}
  \sum_s \psi_{\tau s} \psi_{\tau' s}=D \delta_{\tau \tau'},\qquad \sum_\tau \psi_{\tau r} \psi_{\tau s}=D \delta_{rs}.
\end{equation}
In this basis, the teacher-student overlap $R^{km}$~\eqref{eq:R} is given by
\begin{equation}
  R^{km}=\frac{b}{D} \sum_\tau \Gamma_\tau^k \tilde \omega_\tau^m,
\end{equation}
where we have introduced the projections
\begin{equation}
  \label{eq:Gamma_tau}
  \Gamma_\tau^k=\ussqR \sum_r S_r^k \psi_{\tau r}
\end{equation}
and
\begin{equation}
  \label{eq:omega-tau}
  \tilde{\omega}_\tau^m=\ussqR \sum_r \tilde w_r^m \psi_{\tau r}.
\end{equation}
Since $\tilde{\omega}_\tau^m$ is a static variable, the time evolution of
$\Gamma_\tau^k$ is given by
\begin{equation}
  {\left(\Gamma_\tau^k\right)}_{\mu+1}-{\left(\Gamma_\tau^k\right)}_\mu
  =-\frac{\eta}{\sqrt{\delta}N }v^k \sum_{r} \psi_{\tau r} \left[ \sum_{j \neq
      k}^K v^j \mathcal{A}^{jk}_r+ v^k \mathcal{B}^k_r  -  \sum_n^M\tilde v^n \mathcal{C}^{nk}_r \right]
\end{equation}
As we aim to compute the remaining sum over $r$, two types of terms appear:
\begin{equation}
  \sum_{r}\psi_{\tau r}\EE\left[ \lambda^k \beta_r\right] =
  \frac{1}{\sqrt{\delta}}\left((c-b^2) \delta +
    b^2 \rho_\tau \right)
  \Gamma_\tau^k = \frac{d_\tau}{\sqrt{\delta}} \Gamma_\tau^k,
\end{equation}
where we have defined $d_\tau=(c-b^2) \delta +  b^2 \rho_\tau $, and 
\begin{align}
  \sum_r \psi_{\tau r} \EE\left[  \nu^n \beta_r \right] = \frac{b}{\sqrt{\delta}} \rho_\tau \tilde{\omega}^n_\tau.
\end{align}
Putting everything together, the final evolution of $\Gamma_\tau^k$ is
\begin{equation}
  \begin{split}
    \label{eq:eom-Gamma}
    {\left(\Gamma_\tau^k\right)}_{\mu+1}-{\left(\Gamma_\tau^k\right)}_\mu
    =-\frac{\eta}{\delta N} v^k & \left[d_\tau \Gamma_\tau^k \sum_{j\neq k} v^j
      \frac{ Q^{jj} \EE\left[ g'(\lambda^k)\lambda^k g(\lambda^j) \right]
        -Q^{kj} \EE\left[ g'(\lambda^k)\lambda^j g(\lambda^j)\right]}
      {Q^{kk}Q^{jj}-(Q^{kj})^2}\right.  \\
    &\qquad + \sum_{j\neq k}v^j d_\tau \Gamma_\tau^j \frac{ Q^{kk} \EE\left[
        g'(\lambda^k) \lambda^j g(\lambda^j) \right] - Q^{kj} \EE\left[
        g'(\lambda^k) \lambda^k g(\lambda^j)\right] }
    {Q^{kk}Q^{jj}-(Q^{kj})^2} \\
    &\qquad + d_\tau v^k \Gamma_\tau^k \frac{1}{Q^{kk}}
    \EE\left[  g'(\lambda^k)\lambda^k g(\lambda^k)\right]  \\
    &\qquad -d_\tau \Gamma_\tau^k \sum_{n}\tilde v^n \frac{ T^{nn} \EE\left[
        g'(\lambda^k)\lambda^k\tilde g(\nu^n) \right] -R^{kn} \EE\left[
        g'(\lambda^k) \nu^n \tilde g(\nu^n) \right] }
    {Q^{kk}T^{nn}-(R^{kn})^2} \\
    &\qquad \left.  - b \rho_\tau \sum_{n}\tilde v^n \tilde{\omega}_\tau^n \frac{Q^{kk}
        \EE\left[ g'(\lambda^k) \nu^n \tilde g(\nu^n)\right] -R^{kn} \EE\left[
          g'(\lambda^k) \lambda^k \tilde g(\nu^n)\right]}
      {Q^{kk}T^{nn}-(R^{kn})^2} \right].
  \end{split}
\end{equation}
At this point, we note that the remaining averages appearing in this expression,
such as $\EE\left[\lambda^k g'(\lambda^k)\tilde g(\nu^m) \right]$, depend only
on subsets of the local fields $\{\lambda^{k=1,\ldots,K},
\nu^{m=1,\ldots,M}\}$. As discussed above, the GEP guarantees that these random
variables follow a multi-dimensional Gaussian distribution, so these averages
depend only on the covariances of the local fields~$R^{km}$,
$Q^{k\ell}$, 
and~$T^{mn}$. To simplify the subsequent equations, we will use the 
shorthand notation for the three-dimensional Gaussian averages
\begin{equation}
  I_3(k, j, n) \equiv \EE\left[ g'(\lambda^k) \lambda^j
    \tilde{g}(\nu^n) \right] \, ,
\end{equation}
which was introduced by Saad~\& Solla~\cite{Saad1995a} and that we discuss in
the main text. To make this section self-contained, we recall that arguments
passed to $I_3$ should be translated into local fields on the right-hand side by
using the convention where the indices $j,k,\ell,\iota$ always refer to student
local fields $\lambda^j$, etc., while the indices $n,m$ always refer to teacher
local fields $\nu^n$, $\nu^m$. Similarly,
\begin{equation}
  \label{eq:sm_I3}
  I_3(k, j, j) \equiv \EE\left[ g'(\lambda^k) \lambda^j g(\lambda^j) \right],
\end{equation}
where having the index $j$ as the third argument means that the third factor is
$g(\lambda^j)$, rather than $\tilde{g}(\nu^m)$ in Eq.~\eqref{eq:sm_I3}. The average
in Eq.~\eqref{eq:sm_I3} is taken over a three-dimensional normal distribution with
mean zero and covariance matrix
\begin{equation}
  \label{eq:sm_Phi3}
  \Phi^{(3)}(k, j, n) = \begin{pmatrix}
    \EE\left[ \lambda^k \lambda^k\right] &  \EE\left[ \lambda^k\lambda^j \right] & \EE\left[
    \lambda^k\nu^n \right] \\
    \EE\left[  \lambda^k\lambda^j \right] &  \EE\left[ \lambda^j\lambda^j \right] & \EE\left[ \lambda^j\nu^n \right] \\
    \EE\left[ \lambda^k\nu^n \right] &  \EE\left[ \lambda^j\nu^n \right] & \EE\left[ \nu^n\nu^n \right]
  \end{pmatrix} = \begin{pmatrix}
    Q^{kk} &  Q^{kj} & R^{kn} \\
    Q^{kj} &  Q^{jj} & R^{jn} \\

    R^{kn} &  R^{jn} & T^{nn}
  \end{pmatrix}.
\end{equation}

\subsection{Dynamics of the teacher-student overlap $R^{km}$}

We are now in a position to write the update equation for
\begin{equation}
  \left(R^{km}\right)_{\mu+1} - \left(R^{km}\right)_\mu =\frac{b}{D} \sum_\tau \left[{\left(\Gamma_\tau^k\right)}_{\mu+1}-{\left(\Gamma_\tau^k\right)}_\mu\right] \tilde \omega_\tau^m,
\end{equation}
where we have used that the $\tilde\omega_\tau^m$ are static. When performing
the last remaining sum over $\tau$, two types of terms appear. First, there is
\begin{equation}
  \tilde T^{mn} \equiv \usR\sum_\tau \rho_\tau \tilde \omega_\tau^m\tilde \omega_\tau^n.
\end{equation}
which depends only on the choice of the feature matrix $F_{ir}$ and the teacher
weights $w^*_{mr}$ and is thus a constant of the motion. The second type of term
is of the form
\begin{equation}
  \usR\sum_\tau \rho_\tau \Gamma_\tau^\ell\tilde \omega_\tau^n.
\end{equation}
This sum cannot be reduced to a simple expression in terms of other order
parameters. Instead, we are led to introduce the density
\begin{equation}
  \label{eq:r}
  r^{km}(\rho)=\frac{1}{\varepsilon_\rho} \usR \sum_\tau \Gamma_\tau^k
  \tilde\omega_\tau^m \; \ind\left(\rho_\tau \in
    \mathopen[\rho,\rho+\varepsilon_\rho\mathclose[ \right),
\end{equation}
where $\ind(\cdot)$ is the indicator function which evaluates to~1 if the
condition given to it as an argument is true, and which otherwise evaluates to
0. We take the limit $\varepsilon_\rho\to 0$ after the thermodynamic limit. Then
we can rewrite the order parameter $R^{km}$ as an integral over the density
$r^{km}$, weighted by the distribution of eigenvalues of the operator
$\Omega_{rs}$, $p_{\Omega}(\rho)$:
\begin{equation}
  \label{eq:sm_R_int}
  R^{km}= b \;\int \dd \rho\;  p_{\Omega}(\rho)\;  r^{km}(\rho)\, .
\end{equation}
If, for example, we take the elements of the feature matrix $F_{ir}$ to be
element-wise i.i.d.\ from the normal distribution with mean zero and unit
variance, then the limiting density of eigenvalues of $\Omega$ is given by the
Marchenko-Pastur law~\cite{marchenko1967distribution}:
\begin{equation}
  \label{eq:pmp}
  p_{\mathrm{MP}}(\rho)=\frac{1}{2\pi
    \delta}\frac{\sqrt{(\rho_{\max}-\rho)(\rho-\rho_{\min})}}{\rho}\, ,
\end{equation}
where $\rho_{\min}=\left(1-\sqrt{\delta}\right)^2$ and
$\rho_{\max}=\left(1+\sqrt{\delta}\right)^2$.

The update equation of $r^{km}(\rho)$ can be obtained immediately by
substituting the update equation for $\Gamma_\tau^k$~\eqref{eq:eom-Gamma} into
its definition~\eqref{eq:r}. Finally, in the thermodynamic limit, the
normalised number of steps $t=P / N$ can be interpreted as a continuous
time-like variable, and so we have
\begin{equation}
  R^{km}(t)= b \;\int \dd \rho\;  p_{\Omega}(\rho)\;  r^{km}(\rho, t)
\end{equation}
and we recover the equation of motion for $r^{km}(\rho)$, which we re-state here
for ease of reading:
\begin{align}
  \begin{split}
    \frac{\partial r^{km}(\rho, t)}{\partial t} = -\frac{\eta}{\delta} v^k &
    \left( d(\rho) r^{km}(\rho) \sum_{j\neq k}v^j \frac{ Q^{jj}\; I_3(k, k, j) -Q^{kj}
        I_3(k, j, j)}
      {Q^{jj}Q^{kk}-(Q^{kj})^2}\right.  \\
    &\qquad + d(\rho) \sum_{j\neq k} v^j r^{j m}(\rho) \frac{ Q^{kk} I_3(k, j, j) - Q^{kj}\;
      I_3(k, k, j) }
    {Q^{jj}Q^{kk}-(Q^{kj})^2} \\
    &\qquad + v^k d(\rho) r^{k m}(\rho) \frac{1}{Q^{kk}} I_3(k, k, k)  \\
    &\qquad - d(\rho) r^{km}(\rho) \sum_{n} \tilde v^n \frac{T^{nn} I_3(k, k, n) - R^{kn} I_3(k, n,
      n) }
    {Q^{kk}T^{nn}-(R^{kn})^2} \\
    &\qquad \left.  - b \rho \sum_{n}\tilde v^n \tilde T^{nm}
      \frac{Q^{kk} I_3(k, n, n) -R^{kn} I_3(k, k, n)} {Q^{kk}T^{nn}-(R^{kn})^2}
    \right),
  \end{split}
\end{align}
where $d(\rho)=(c-b^2) \delta + b^2 \rho$. Note that while we have dropped the
explicit time dependence from the right-hand side to keep the equation readable,
all the order parameters on the right-hand side are explicitly time-dependent,
i.e. $Q^{jj}=Q^{jj}(t)$, $r^{km}(\rho)=r^{km}(\rho, t)$, and the averages
$I_3(\cdot)$ are also time-dependent via their dependence on the order
parameters (see Eq.~\eqref{eq:sm_I3} and the subsequent discussion). In order to
close the equations of motion, we now need to find the equations for the order
parameters that are quadratic in the weights.

\subsection{Order parameters that are quadratic in the weights}

There are two order parameters that appear when evaluating the covariance of the
$\lambda$ variables:
\begin{equation}
  Q^{k\ell}\equiv\EE\left[ \lambda^k \lambda^\ell\right] = \left[c-b^2\right]W^{k\ell}+ b^2 \Sigma^{k \ell}.
\end{equation}
We will look at both $W^{k\ell}$ and $\Sigma^{k\ell}$ in turn now.

\paragraph{Equation of motion for $W^{k\ell}$} For the student-student overlap
matrix
\begin{equation}
  W^{k\ell} = \usN \sum_i^N w_i^k w_i^\ell,
\end{equation}
we find, after some algebra, that updates read
\begin{align}
  \begin{split}
    \left(W^{k\ell}\right)^{\mu+1}  - \left(W^{k\ell}\right)_\mu &=
    - \frac{\eta}{N^{3/2}} v^k \sum_i^N w_i^\ell \left[\sum_{j\neq k}^K v^j \cala_i^{jk}
      + v^k \calb_i^k - \sum_n^M\tilde v^n \calc_i^{nk}\right] \\
    & \quad - \frac{\eta}{N^{3/2}} v^\ell \sum_i^N w_i^k
    \left[\sum_{j\neq \ell}^K v^j \cala_i^{j\ell} + v^\ell \calb_i^\ell
      - \sum_n^M\tilde v^n \calc_i^{n\ell}\right] \\
    & \quad + \frac{\eta^2}{N^2} v^k v^\ell \sum_i^N {f(u_i)}^2 g'(\lambda^k) g'(\lambda^\ell) 
    \left[ \sum_{j,\iota}^K v^j v^\iota g(\lambda^j)g(\lambda^\iota) + \sum_{n,m}^M 
      \tilde v^n \tilde v^m \tilde g(\nu^n)\tilde g(\nu^m) \right. \\
      & \hspace*{16.5em} \left. - 2 \sum_j^K \sum_m^M v^j \tilde v^m
        g(\lambda^j) \tilde g(\nu^m)\right]
  \end{split}
\end{align}
For the terms linear in the learning rate $\eta$, we can immediately carry out
the sum over $i$, which yields terms of the type
\begin{equation}
  \ussqN \sum_i w^\ell_i \EE\left[ g(\lambda^j) g'(\lambda^k) f(u_i) \right] =
  \EE\left[ g'(\lambda^k) \lambda^\ell g(\lambda^j) \right] = I_3(k, \ell, j)
  \quad \text{etc.}
\end{equation}
The term quadratic in the learning rate $\eta$ requires the evaluation of terms
of the type
\begin{equation}
  \usN \sum_i \EE\left[ f(u_i)^2  g'(\lambda^k) g'(\lambda^\ell)
  g(\lambda^j)g(\lambda^\iota) \right]
  = c\EE\left[ g'(\lambda^k) g'(\lambda^\ell) g(\lambda^j)g(\lambda^\iota) \right].
\end{equation}
The sum over $i$ thus makes this second-order term contribute to the total
variation of $W^{k\ell}$ at leading order, and we're left with an average over
four local fields, for which we introduce the short-hand
\begin{equation}
  \label{eq:sm_I4}
  I_4(k, \ell, j, \iota) \equiv \EE\left[ g'(\lambda^k)
    g'(\lambda^\ell) g(\lambda^j) g(\lambda^\iota)\right],
\end{equation}
where we use the same notation as we did for $I_3(\cdot)$~\eqref{eq:sm_I3}. The
full equation of motion for $W^{k\ell}$ thus reads
\begin{align}
  \begin{split}
    \frac{\dd W^{k\ell}(t)}{\dd t} = &- \eta v^k \left(\sum_{j}^K v^j I_3(k, \ell, j) -
      \sum_n \tilde v^n I_3(k, \ell, n) \right) - \eta v^\ell \left(\sum_{j}^K
      v^j I_3(\ell, k, j) - \sum_n\tilde v^n I_3(\ell, k, n) \right)\\
    & + c\eta^2 v^k v^\ell \left( \sum_{j,a}^K v^j v^a I_4(k, \ell, j, a) - 2 \sum_j^K \sum_m^M
      v^j \tilde v^m I_4(k, \ell, j, m) + \sum_{n,m}^M \tilde v^n \tilde v^m I_4(k, \ell, n, m) \right).
  \end{split}
\end{align}

\paragraph{Equation of motion for $\Sigma^{k\ell}$} After rotating to the basis
$\psi_\tau$, we have
\begin{equation}
  \label{eq:Sigma_rotated}
  \Sigma^{k\ell} \equiv \frac{1}{D} \sum_r S_r^k S_r^\ell = \frac{1}{D}
  \sum_{\tau} \Gamma_\tau^k \Gamma_\tau^\ell.
\end{equation}
It is then immediate that
\begin{align}
  \begin{split}
    {(\Sigma^{k\ell})}^{\mu+1} - {(\Sigma^{k\ell})}_\mu & = \frac{1}{D} \sum_{\tau}  {\left(\Gamma_\tau^\ell\right)}_\mu \left[ {(\Gamma_\tau^k)}^{\mu+1} - {(\Gamma_\tau^k)}_\mu \right]
    + \frac{1}{D} \sum_{\tau} \left(\Gamma_\tau^k\right)_\mu \left[
      {(\Gamma_\tau^\ell)}^{\mu+1} - {(\Gamma_\tau^\ell)}_\mu \right] \\
    & \qquad + \frac{\eta^2}{D^2N} \sum_\tau \sum_{r, s}^R \psi_{\tau r}\psi_{\tau s} \EE\left[
    \Delta^2 g'(\lambda^k) g'(\lambda^\ell) \beta_r \beta_s \right].
  \end{split}
\end{align}
The terms linear in $\eta$ can be obtained directly by substituting in the
update equation for $\Gamma_\tau^k$~\eqref{eq:eom-Gamma} and are similar to the
update equation for $r^{km}(\rho)$. As for the term quadratic in $\eta$, we have
to leading order
\begin{multline}
  \frac{\eta^2}{DN} \sum_{r, s}^R \psi_{\tau r}\psi_{\tau s} \EE\left[ \Delta^2
  g'(\lambda^k) g'(\lambda^\ell) \beta_r \beta_s \right] = \frac{\eta^2}{DN}
  \sum_{r}^R (\psi_{\tau r})^2 \EE\left[ \Delta^2 g'(\lambda^k)
  g'(\lambda^\ell)\right]\EE\left[ \beta_r^2 \right]\\
  = \frac{\eta^2}{N} \EE\left[ \Delta^2 g'(\lambda^k)
  g'(\lambda^\ell)\right]\left[ (c-b^2)\rho_\tau + \frac{b^2}{\delta}\rho_\tau^2 \right],
\end{multline}
where we have used that covariance of $\beta_r$ is given by
\begin{equation}
  \label{eq:cov-betar}
  \EE\left[ \beta_r^2\right] = c - b^2 + \frac{b^2}{\delta} \sum_t \Omega_{rt}^2.
\end{equation}
To deal with the remaining sum over $\tau$, we again make use of the fact that
the equation of motion for $\Sigma^{k\ell}$ depends on the eigenvector index
$\tau$ only through the eigenvalue $\rho_\tau$. Introducing the density
\begin{equation}
  \sigma^{k\ell}(\rho)=\frac{1}{\varepsilon_\rho}\; \usR \sum_\tau \Gamma_\tau^k \Gamma_\tau^\ell
  \ind\left(\rho_\tau \in \mathopen[\rho,\rho+\varepsilon_\rho\mathclose[\right),
\end{equation}
as we did for $r^{km}(\rho)$~\eqref{eq:r}, we have
\begin{equation}
  \Sigma^{k\ell}(t)= \;\int \dd \rho\;  p_{\Omega}(\rho)\;  \sigma^{k\ell}(\rho, t)
\end{equation}
with
\begin{align}
  \begin{split}
    \frac{\partial \sigma^{k\ell}(\rho, t)}{\partial t} = -\frac{\eta}{\delta} &
    \left(d(\rho) v^k \sigma^{k\ell}(\rho) \sum_{j\neq k}v^j \frac{Q^{jj}
        I_3(k, k, j) -Q^{kj} I_3(k, j, j)}
      {Q^{jj}Q^{kk}-(Q^{kj})^2}\right.  \\
    &\qquad + v^k \sum_{j\neq k} v^j d(\rho) \sigma^{j \ell}(\rho) \frac{ Q^{kk}
      I_3(k, j, j) - Q^{kj} I_3(k, k, j) }
    {Q^{jj}Q^{kk}-(Q^{kj})^2} \\
    &\qquad + d(\rho) v^k \sigma^{k \ell}(\rho)v^k \frac{1}{Q^{kk}} I_3(k, k, k)  \\
    &\qquad - d(\rho) v^k \sigma^{k\ell}(\rho) \sum_n \tilde v^n\frac{T^{nn}
      I_3(k, k, n) - R^{kn} I_3(k, n, n) }
    {Q^{kk}T^{nn}-(R^{kn})^2} \\
    &\qquad -b \rho v^k \sum_n \tilde v^n r^{\ell n}(\rho) \frac{Q^{kk} I_3(k,
      n, n) -R^{kn}
      I_3(k, k, n)}{Q^{kk}T^{nn}-(R^{kn})^2} \\
    & \qquad + \text{all of the above with } \ell\to k, k\to\ell\Bigg).\\
    & \hspace*{-1em} + \eta^2 v^k v^\ell \left[(c - b^2)\rho +
      \frac{b^2}{\delta}\rho^2\right] \left(
      \sum_{j,\iota}^K v^j v^\iota I_4(k, \ell, j, \iota)\right. \\
    & \hspace*{12.5em} \left. - 2 \sum_j^K \sum_m^M v^j \tilde{v}^m I_4(k, \ell,
      j, m) + \sum_{n,m}^M \tilde{v}^n\tilde{v}^m I_4(k, \ell, n, m) \right)
  \end{split}
\end{align}

\subsection{Second-layer weights}

Finally, we will treat each of the second-layer weights of the student $\vv$ as
an order parameter in its own right. Their equations of motion are readily found
from from their SGD udpate~\eqref{eq:sgdv}, and read
\begin{equation}
  \frac{\dd v^k}{\dd t} = \eta \left[ \sum_n^M \tilde v_n I_2(k, n) - \sum_j^K
    v^j I_2(k, j) \right],
\end{equation}
where we have introduced the final short-hand
\begin{equation}
  \label{eq:supp_I2}
  I_2(k, j) \equiv \EE\left[ g(\lambda^k) g(\lambda^j)\right] \; \text{etc.}
\end{equation}
where we again use the notation we introduced for $I_3(\cdot)$~\eqref{eq:sm_I3}. 

\subsection{Generalisation error}

Having introduced the short-hand for the integrals
$I_2(k, j)$~\eqref{eq:supp_I2}, we realise that their form also enters the
formula for the generalisation error
\begin{align}
  \epsilon_g(\vtheta, \widetilde{\vtheta}) =  \frac{1}{2}\EE {\left( \sum_{k}^K
      v^k g(\lambda^k) - \sum_{m}^M \tilde v^m \tilde g(\nu^m)\right)}^2 =
  \frac{1}{2} \sum_{k,\ell} v^k v^\ell I_2(k, \ell) + \frac{1}{2} \sum_{n,m}
  \tilde v^n \tilde v^m I_2(n, m) -
  \sum_{k,n} v^k \tilde v^n I_2(k, n).
\end{align}
For example, for a student with
$g(\lambda^k)=\erf(\lambda^k/\sqrt{2})$ and a teacher with
$\tilde g(\nu^m)=\max(0, \nu^m)$, we find that
\begin{multline}
  \label{eq:eg-example} \epsilon_g(Q^{k\ell}, R^{kn}, T^{nm}, v^k, \tilde v^m) =
  \frac{1}{\pi} \sum_{k,\ell} v^k v^\ell \arcsin \frac{Q^{k\ell}}{\sqrt{1
        + Q^{kk}}\sqrt{1 + Q^{\ell\ell}}} - \sum_{k,n} v^k \tilde v^n
  \frac{R^{kn}}{\sqrt{2 \pi } \sqrt{1 + Q^{kk}}}\\ + \frac{1}{4\pi} \sum_{n,m}
  \tilde{v}^n \tilde{v}^m \left(\sqrt{T^{mm}
      T^{nn}-\left(T^{nm}\right)^2}+T^{nm}\left[\pi +
      2\arctan\frac{T^{nm}}{\sqrt{T^{mm}
            T^{nn}-\left(T^{nm}\right)^2}}\right]\right)
\end{multline}

\subsection{Summary of the derivation}

We have now completed the programme that we embarked upon at the beginning of
this Appendix~\ref{app:derivation}: we have derived a closed set of equations of
motion for the teacher-student overlap $R^{km}$
(\ref{eq:sm_R_int},\ref{eq:eom-r}) the student-student overlap
$Q^{k\ell}= \left[c - b^2\right]W^{k\ell}+ b^2\Sigma^{k \ell}$
(\ref{eq:eom-W},\ref{eq:Sigma_int},\ref{eq:eom-sigma}), and the student's
second-layer weights $v^k$~\eqref{eq:eom-v}. These equations give us complete
access to the dynamics of a neural network performing one-shot stochastic
gradient descent on a data set generated by the hidden manifold model. We can
now integrate these equations and substitute the values of the order parameters
at any time into the expression for the generalisation
error~\eqref{eq:eg-order-parameters}, thereby tracking the dynamics of the
generalisation error at all times. We describe this procedure in more detail
next.

\subsection{Explicit form of the integrals $I_3$ and $I_4$ for sigmoidal
  students}

The explicit forms of the integrals $I_3$ and $I_4$ that appear in the equations
of motion for the order parameters and the generalisation error for networks
with $g(x)=\tilde g(x)=\erf\left( x/\sqrt{2} \right)$ were first given
by~\cite{Biehl1995,Saad1995a}. Here, we will state them to make the paper as
self-contained as possible. Denoting the elements of the covariance matrix such
as $\Phi^{3}$~\eqref{eq:sm_Phi3} as~$\phi_{ij}$, we have
\begin{equation}
  I_3(\cdot, \cdot, \cdot) = 
  \frac{2}{\pi}\frac{1}{\sqrt{\Lambda_3}} \frac{\phi_{23}(1 + \phi_{11}) -
    \phi_{12}\phi_{13}}{1 + \phi_{11}}
\end{equation}
with
\begin{equation}
  \Lambda_3 = (1 + \phi_{11}) (1 + \phi_{33}) - \phi_{13}^2.
\end{equation}
For the average~$I_4$, we have a covariance matrix $\Phi^{(4)}$ that is
populated in analogy to $\Phi^{(3)}$~\eqref{eq:sm_Phi3}, we have
\begin{equation}
  I_4(\cdot, \cdot, \cdot, \cdot) = \frac{4}{\pi^2}\frac{1}{\sqrt{\Lambda_4}}\arcsin\left( \frac{\Lambda_0}{\sqrt{\Lambda_1\Lambda2}} \right)
\end{equation}
where
\begin{align}
  \Lambda_4 &= (1 + \phi_{11})  (1 + \phi_{22}) - \phi_{12}^2,\\[0.5em]
  \Lambda_0 &= \Lambda_4 \phi_{34} - \phi_{23}\phi_{24}(1 + \phi_{11}) - \phi_{13}\phi_{14}(1 +
  \phi_{22}) + \phi_{12}\phi_{13}\phi_{24} + \phi_{12}\phi_{14}\phi_{23},\\[0.5em]
  \Lambda_1 &= \Lambda_4 (1 + \phi_{33}) - \phi_{23}^2(1 + \phi_{11}) - \phi_{13}^2(1 +
                \phi_{22}) + 2 \phi_{12}\phi_{13}\phi_{23},\\[0.5em]
  \Lambda_2 & = \Lambda_4 (1 + \phi_{44}) - \phi_{24}^2(1 + \phi_{11}) - \phi_{14}^2(1 +
                \phi_{22}) + 2 \phi_{12}\phi_{14}\phi_{24}
\end{align}

\section{The equations of motion do not close in the trivial basis}
\label{sec:cannot-close}

Here we give a short demonstration that it is not possible to close the
equations for order parameters if we do not rotate their dynamics to the basis
given by the eigenvectors of $\Omega$, which is what we do in our derivation in
Sec.~\ref{sec:sgd}.

\subsection{Order parameters that are linear in the weights}

To start with a variable that is linear in the weights, take the time-evolution
of $S_r^k$.  It is clear that the tensor structure of the
result~\eqref{eq:eom-S} will be of the form
\begin{equation}
  (S_r^k)^{\mu+1}-(S_r^k)^\mu = -\frac{\eta}{N} \left[ \sum_\ell D^{k
      \ell}\sum_s \Omega_{rs}S_s^\ell
    +  \sum_m E^{km}\sum_s \Omega_{rs} \tilde w_s^m \right]
\end{equation}
where $D^{k \ell}$ and $E^{km} $ are known quantities, expressed in terms of the
matrices $Q, T, R$, and we have introduced the operator
\begin{equation}
\Omega_{rs}=\usN \sum_i F_{ir} F_{is}
\end{equation}
which has diagonal elements equal $1$, and off diagonal elements of
order $1/\sqrt{N}$.

In particular we can use this evolution to study the evolution of $R$:
\begin{equation}
  (R^{km})^{\mu+1}-(R^{km})^\mu=-\langle u f(u)\rangle \frac{\eta}{N} \left[
    \sum_\ell D^{k \ell} \usR\sum_{rs} \tilde w_r^m\Omega_{rs}S_s^\ell + 
    \sum_m E^{km}\usR \sum_{rs} \tilde w_s^r \Omega_{rs} \tilde w_s^m  \right]
\end{equation}
The point of this analysis is to show that the time evolution of $S_r^k$
involves ${(\Omega S)}_r^\ell$. Therefore to know the evolution of $S$ we need
the one of $\Omega S$. This is not inocuous because, in order to have dynamical
evolution equations with only ``up'' indices, we need to contract it.  The
evolution of $R^{km}$, which is proportional to the scalar product (in the
$R$-dimensional manifold space) of $S^k$ and $\tilde w^m$, is thus given by the
scalar product of $\Omega S^k$ and $\tilde w^m$.

It is not difficult to see that the evolution of $\Omega S$ will require knowing
$\Omega^2 S$ etc. So we have an infinite hierarchy of coupled equations, which
would be hard to analyse. Yet, we can find closed equations by changing basis
for $S$.

\section{Additional details on the numerical experiments in
  Sec.~\ref{sec:complexity}}
\label{sec:supp_complexity}

\begin{figure*}
  \includegraphics[width=\textwidth]{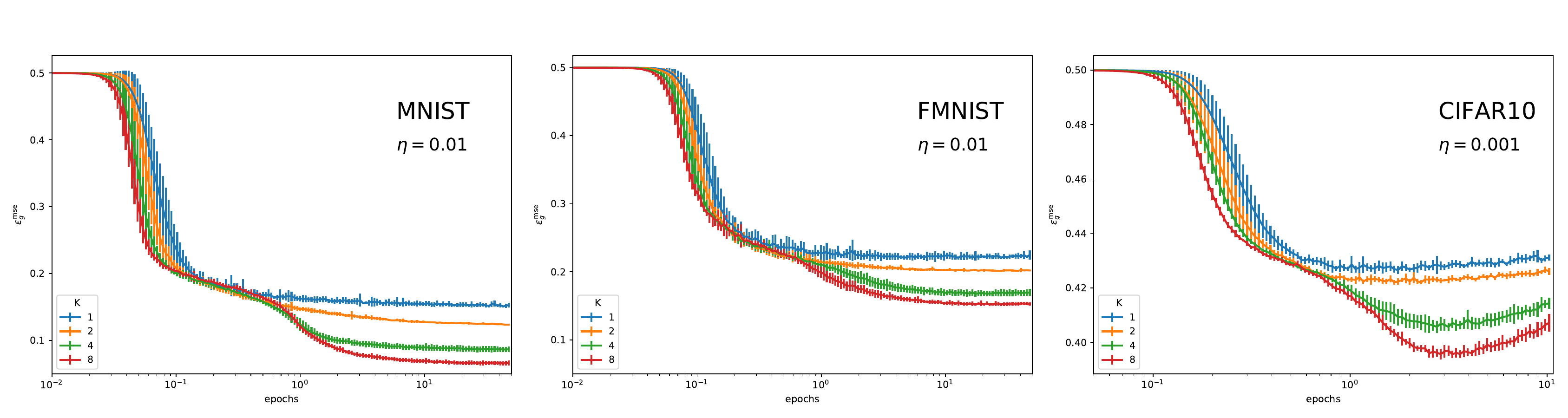}
  \caption{\label{fig:supp_inc_comp_sigmoidal} \textbf{Two-layer sigmoidal
      neural networks learn functions of increasing complexity on different data
      sets.}  We plot the mean-squared error as a function of training time for
    sigmoidal networks with increasing hidden layer when trained on three
    different data sets. The curves are obtained by averaging ten different runs,
    starting from different initial weights. Error bars indicate one standard
    deviation. For all plots,
    $g(x)=\mathrm{erf}\left(x/\sqrt{2}\right)$, gaussian initial
    weights with std.\ dev.\ $10^{-3}$, batch size 32.}
\end{figure*}
For the experiments demonstrating the learning of functions of increasing
complexity discussed in Sec.~\ref{sec:complexity}, we constructed data sets for
binary classification by splitting the image data sets as follows:
\begin{description}
\item[MNIST] even vs.\ odd numbers
\item[Fashion MNIST] t-shirt/top, pullover, dress, sandal and bag vs.\ trouser, coat,
  shirt, sneaker, ankle boot
\item[CIFAR10] air plane, bird, deer, frog, ship vs.\ automobile, cat, dog,
  horse and truck.
\end{description}

We first demonstrate in Fig.~\ref{fig:supp_inc_comp_sigmoidal} that sigmoidal
networks show the same learning of functions of increasing complexity discussed
in Sec.~\ref{sec:complexity} for CIFAR10 when trained on MNIST or FMNIST. Note
that for CIFAR10 in particular, we see the effects of over-training set in after
several epochs, when the generalisation error starts to increase again (we use
plain SGD without any explicit regularisation in these experiments).

\begin{figure}
  \includegraphics[width=.66\linewidth]{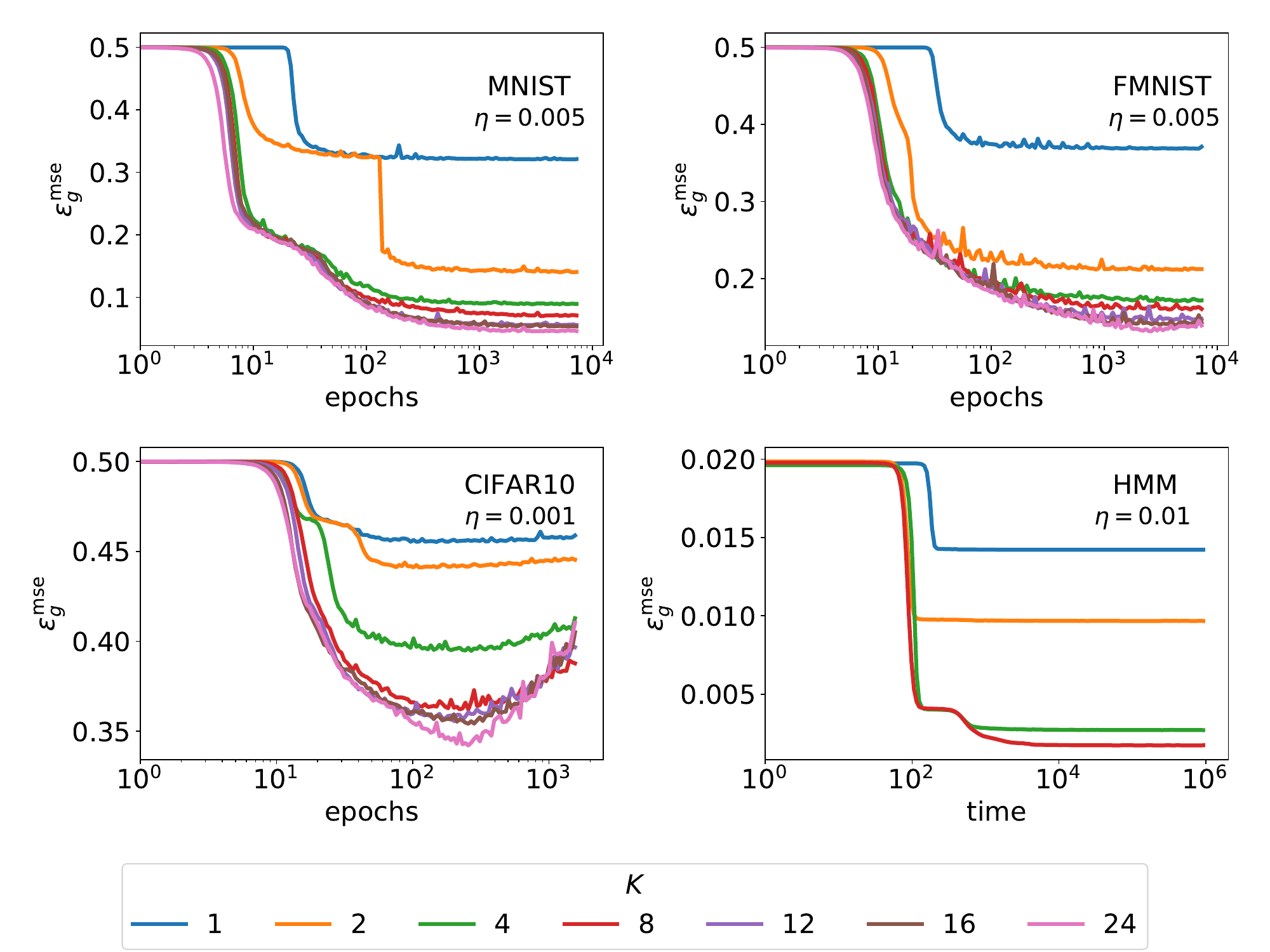}
  \caption{\label{fig:supp_inc_comp_relu} \textbf{Two-layer ReLU neural networks
      also learn functions of increasing complexity, show more run-to-run
      variance.}  We plot the mean-squared error as a function of training time
    for sigmoidal networks with increasing number of nodes $K$ when trained on
    three different data sets. For all plots, $g(x)=\max(0, x)$, gaussian
    initial weights with std.\ dev.\ $10^{-3}$, batch size 32. For online
    learning, we chose a teacher with
    $\tilde g(x)=\max(x, 0), M=10, \tilde v^m=\nicefrac{1}{M}$.}
\end{figure}

We also repeated these experiments for ReLU networks with activation function
$g(x)=\max(0, x)$. While the dynamics of ReLU students also show a progression
from simple to more complex classifiers, the run-to-run fluctuations are much
larger than for the sigmoidal students. This is true both quantitatively, but
also qualitatively: for example, networks sometimes get stuck in really
sub-optimal minimisers for a long time. Hence, plotting the mean trajectories is
not as informative as the standard variations would be very high, so in
Fig.~\ref{fig:supp_inc_comp_relu} we instead show representative curves for
individual runs of ReLU students for all three data sets and for online learning
from a teacher with $\tilde g(x)=\max(x, 0), M=10, \tilde v^m=\nicefrac{1}{M}$.
\end{widetext}

\end{document}